\documentclass{article}
\usepackage{kr}

\usepackage{times}
\usepackage{soul}
\usepackage{url}
\usepackage{xr-hyper}
\usepackage[hidelinks]{hyperref}
\usepackage[utf8]{inputenc}
\usepackage[small]{caption}
\usepackage{graphicx}
\usepackage{amsmath}
\usepackage{amsthm}
\usepackage{amssymb}
\usepackage{booktabs}
\usepackage{multicol}
\usepackage{multirow}
\usepackage{pifont}
\usepackage{algorithm}
\usepackage[noend]{algpseudocode}
\usepackage{tikz}
\usetikzlibrary{arrows.meta, positioning, calc, fit}
\definecolor{intendedcolor}{RGB}{50,50,50}
\definecolor{shortcutcolor}{RGB}{200,50,50}
\definecolor{comp1color}{RGB}{70,130,180}
\definecolor{comp2color}{RGB}{200,100,50}

\usepackage{newfloat}
\usepackage{listings}
\DeclareCaptionStyle{ruled}{labelfont=normalfont,labelsep=colon,strut=off} % DO NOT CHANGE THIS
\floatstyle{ruled}
\newfloat{listing}{tb}{lst}{}
\floatname{listing}{Listing}

\newtheorem{definition}{Definition}
\newtheorem{theorem}{Theorem}

\newtheorem{proposition}{Proposition}

\newtheorem{example}{Example}

\theoremstyle{remark}
\newtheorem{remark}{Remark}
\newtheorem{observation}{Observation}

\title{Constraint-Based Analysis of Reasoning Shortcuts\\in Neurosymbolic Learning}

\author{%
Akihiro Takemura$^1$\and
Katsumi Inoue$^1$\and
Masaaki Nishino$^2$
\affiliations
$^1$National Institute of Informatics, Tokyo, Japan\\
$^2$Communication Science Laboratories, NTT, Inc., Japan\\
\emails
\{atakemura, inoue\}@nii.ac.jp,
masaaki.nishino@ntt.com
}

\begin{document}

\maketitle

\begin{abstract}
Neurosymbolic systems can satisfy logical constraints during learning without achieving the intended concept-label correspondence; this is a problem known as reasoning shortcuts.
We formalize reasoning shortcuts as a constraint satisfaction problem and investigate under which conditions concept mappings are uniquely determined by the constraints.
We prove that a discrimination property (requiring that no valid concept mapping can be transformed into another valid mapping by swapping two concept values) is necessary for shortcut-freeness under bijective mappings, but demonstrate via a counterexample that it is insufficient even when the constraint graph is connected.
We develop an ASP-based algorithm that verifies whether a given constraint set uniquely determines the intended concept mapping, with proven soundness and completeness.
When shortcuts are detected, a greedy repair algorithm eliminates them by augmenting the constraint set, converging in at most $k$ iterations, where $k$ is the number of alternative valid mappings.
We further provide a complexity classification: deciding shortcut-freeness is coNP-complete, counting shortcuts is \#P-complete, and finding minimal repairs is NP-hard.
We also establish sample complexity bounds showing that logarithmically many label queries suffice for disambiguation in favorable cases, while querying all ambiguous positions suffices in the worst case.
Experiments across eight benchmark domains validate our approach.
\end{abstract}

\section{Introduction}

Neurosymbolic AI systems integrate neural perception with symbolic reasoning, combining the pattern recognition capabilities of deep learning with the reliability and explainability of symbolic logic \cite{besoldNeuralSymbolicLearningAndReasoning,manhaeveNeuralProbabilisticLogic2021a,yangNeurASPEmbracingNeural2020,badreddineLogicTensorNetworks2022}. 
However, recent work has identified a critical problem called \emph{reasoning shortcuts}, where models satisfy constraints while learning concept mappings that differ from the intended correspondence \cite{marconatoNotAllNeuroSymbolic2023,yangAnalysisAbductiveLearning2024}.

Consider a neurosymbolic system for visual arithmetic: neural networks predict digit concepts from images, and constraints enforce arithmetic relationships (e.g., the labels assigned to two images must sum to 3). 
While the system may achieve high accuracy on training tasks, the learned concepts might be \emph{swapped}, predicting 1 as 2 and vice versa, yet still satisfy all constraints. 
This reasoning shortcut undermines both correctness (concepts do not correspond to their intended labels) and robustness (fails on out-of-distribution tasks) \cite{marconatoBEARSMakeNeuroSymbolic2024a,vankriekenNeurosymbolicReasoningShortcuts2025}.
We focus on the class of constraint-based neurosymbolic systems where logical rules over concept mappings are explicit (Definition~\ref{def:nesy_problem}); our results characterize when such a constraint set uniquely determines the intended concept mapping, and at what computational cost.

Recent empirical work has developed mitigation strategies including uncertainty-aware methods \cite{marconatoBEARSMakeNeuroSymbolic2024a}, theoretical risk analysis \cite{yangAnalysisAbductiveLearning2024}, prototypical neurosymbolic architectures that encourage models to rely on correct internal evidence \cite{andolfiRightRightReasons}, and comprehensive benchmarks \cite{bortolottiNeurosymbolicBenchmarkSuite2024}. 
However, it remains unclear \emph{whether and when constraints uniquely determine concept mappings}, which is essential for deploying neurosymbolic systems with confidence.

We approach this question by formalizing reasoning shortcuts as a \emph{constraint satisfaction problem} (CSP). 
Given a set of constraints $C$ over concept mappings $\phi: N \to S$, we ask, does $C$ admit exactly one valid mapping (shortcut-free), or do multiple valid mappings exist (shortcuts present)? 
This CSP perspective enables us to:
\begin{itemize}
    \item \emph{Prove necessary conditions}: 
    Discrimination (no valid mapping related to another by a concept transposition) is necessary for shortcut-freeness (Theorem~\ref{thm:discrimination_necessary}), but insufficient in general. 
    We construct an explicit counterexample where discrimination holds yet shortcuts persist via longer permutation cycles, even with a connected constraint graph (Example~\ref{cex:discrimination_fails}).
    
    \item \emph{Classify complexity}: Deciding shortcut-freeness is coNP-complete, counting shortcuts is \#P-complete, and finding minimal repairs is NP-hard (Theorems~\ref{thm:complexity_conp}--\ref{thm:complexity_repair}).
    
    \item \emph{Bound disambiguation cost}: We define a shortcut multiplicity measure with connections to the M-unambiguity~\cite{wangLearningLatentModels2023}, and prove that logarithmically many label queries suffice for disambiguation in favorable cases, while querying all ambiguous positions suffices in the worst case (Theorem~\ref{thm:sample_complexity}).
    
    \item \emph{Provide practical algorithms}: a verification algorithm based on \emph{Answer Set Programming} (ASP)~\cite{gebserConflictdrivenAnswerSet2012}, whose native model enumeration makes it well-suited to exhaustive shortcut search, with correctness guarantees (Algorithm~\ref{alg:asp_verification}), greedy repair that eliminates shortcuts by augmenting the constraint set with convergence guarantees (Algorithm~\ref{alg:repair}), and active learning strategies. 
\end{itemize}

Our key observation is that reasoning shortcuts arise from \emph{symmetries in the constraint structure}. 
Disconnected constraint graphs allow independent concept swaps within components (Example~\ref{ex:simple_addition_intro}), while symmetric constraints such as modulo arithmetic admit rotational shortcuts even when the graph is connected (Example~\ref{cex:discrimination_fails}).
This connects our work to classical symmetry-breaking in CSPs \cite{gent2006symmetry}, but extends it to the neurosymbolic setting where a specific ground-truth mapping must be uniquely determined.

Our work complements existing neurosymbolic approaches, e.g., 
shortcut detection methods \cite{marconatoNotAllNeuroSymbolic2023,yangAnalysisAbductiveLearning2024}, 
ASP-based systems \cite{yangNeurASPEmbracingNeural2020,skryaginAnswerSetNetworks2024}, 
Semantic Loss \cite{DBLP:conf/icml/XuZFLB18}, 
and Concept Bottleneck Models (CBMs) \cite{kohConceptBottleneckModels2020}, 
by providing the \emph{formal foundations}: we prove when shortcuts can and cannot exist, characterize their computational complexity, and supply verification and repair algorithms with correctness guarantees. 

\begin{figure}[t]
\centering
\input{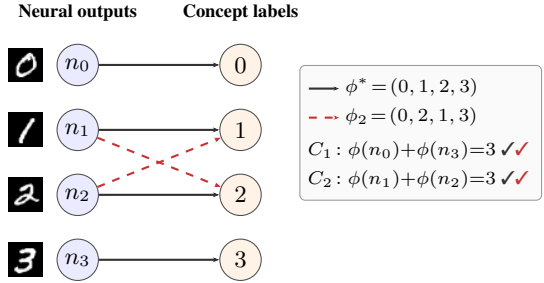}
\caption{Reasoning shortcut in the 4-node addition problem (Example~\ref{ex:simple_addition_intro}). 
Both the intended mapping $\phi^*$ (solid) and the shortcut $\phi_2$ (dashed) satisfy all constraints, but $\phi_2$ incorrectly swaps concepts 1 and 2.}
\label{fig:shortcut_overview}
\end{figure}

\begin{example}[Motivating Example: Simple Addition]
\label{ex:simple_addition_intro}
Consider learning a mapping from 4 neural outputs to 4 digit concepts using addition constraints (Figure \ref{fig:shortcut_overview}):

\emph{Problem setting:}
\begin{itemize}
    \item Neural outputs: $N = \{n_0, n_1, n_2, n_3\}$
    \item Concept labels: $S = \{0, 1, 2, 3\}$
    \item Constraints: \\ 
            $C_1: \phi(n_0) + \phi(n_3) = 3$; $C_2: \phi(n_1) + \phi(n_2) = 3$
    \item Intended concept mapping\\
            $\phi^* = \{n_0 \mapsto 0, n_1 \mapsto 1, n_2 \mapsto 2, n_3 \mapsto 3\}$\\
            or equivalently in tuple notation, $\phi(n_i)$ as the $i$-th entry: 
            $\phi^* = (0, 1, 2, 3)$
\end{itemize}

\emph{Issue:} 
There are multiple bijective mappings that satisfy all constraints. 
For example:
\begin{itemize}
    \item $\phi_1 = (0, 1, 2, 3)$ --- intended mapping
    \item $\phi_2 = (0, 2, 1, 3)$ --- swaps $n_1 \leftrightarrow n_2$
    \item $\phi_3 = (3, 1, 2, 0)$ --- swaps $n_0 \leftrightarrow n_3$
    \item ... (8 total bijective valid mappings)
\end{itemize}

The network may learn any of these 8 mappings during training. 
If the learned mapping differs from $\phi^*$, test-time predictions will be systematically incorrect.

The constraint graph $G_C$ (Definition~\ref{def:constraint_graph}; Figure~\ref{fig:constraint_graphs}) decomposes into two disconnected components $\{n_0, n_3\}$ and $\{n_1, n_2\}$, and this disconnection allows the two component-level transpositions to combine into 8 valid bijections.
\end{example}
This paper also provides a mechanism to: 
(1) detect such shortcuts via ASP-based verification, and
(2) eliminate them by augmenting the constraint set (e.g., adding ``$\phi(n_1) = 1$'' reduces the 8 valid mappings to 4).

Our framework provides immediately deployable tools for practitioners: ASP verification directly checks shortcut-freeness without requiring analytical proofs, repair algorithms systematically eliminate shortcuts with convergence guarantees, and active learning strategies achieve near-optimal label efficiency. 

This paper is organized as follows: 
Section~\ref{sec:related_works} surveys related work.
Section~\ref{sec:preliminaries} introduces the problem setting and core definitions. 
Section~\ref{sec:discrimination} develops theoretical foundations: necessary conditions, their limitations, symmetry analysis, sample complexity bounds, and practical label selection strategies. 
Section~\ref{sec:algorithms} presents verification and repair algorithms with complexity analysis.
Section~\ref{sec:experiments} provides experimental validation across 8 benchmark domains.
Section~\ref{sec:conclusion} concludes with open problems.

\section{Related Works}\label{sec:related_works}

\subsection{Reasoning Shortcuts in Neurosymbolic Systems}

Marconato et al. \shortcite{marconatoNotAllNeuroSymbolic2023} first characterized reasoning shortcuts in neurosymbolic learning, identifying four key conditions under which they arise. 
They demonstrated that models can achieve high task accuracy while learning concept mappings that differ from the intended correspondence, compromising both interpretability and out-of-distribution robustness.
Building on this, Marconato et al. \shortcite{marconatoBEARSMakeNeuroSymbolic2024a} introduced \textit{BEARS}, an uncertainty-aware framework for detecting and mitigating shortcuts through Bayesian estimation. 
Marconato et al. \shortcite{marconatoNeuroSymbolicContinualLearning2023a} further showed that reasoning shortcuts persist across continual learning scenarios, requiring explicit concept rehearsal strategies.

Recent theoretical work by Yang et al. \shortcite{yangAnalysisAbductiveLearning2024} provided the first formal analysis of shortcut risk, quantifying it through knowledge base complexity, sample size, and hypothesis space. 
They proved that Abductive Learning (ABL) algorithms reduce shortcut risk compared to standard neurosymbolic approaches by selecting appropriate distance functions. 
van Krieken et al. \shortcite{vankriekenNeurosymbolicReasoningShortcuts2025} demonstrated that the independence assumption, commonly used in neurosymbolic frameworks to simplify inference, prevents models from correctly representing uncertainty over concept combinations, thereby hindering awareness of reasoning shortcuts.
In a complementary direction, Takemura and Inoue \shortcite{takemuraDifferentiableLogicProgrammingRS} apply matrix-based differentiable logic programming to empirically mitigate shortcuts through one-to-one atom grounding; our work addresses the orthogonal question of when constraints formally determine uniqueness.

To facilitate systematic evaluation, Bortolotti et al. \shortcite{bortolottiNeurosymbolicBenchmarkSuite2024} introduced \textit{RSBench}, a comprehensive benchmark suite providing customizable tasks affected by reasoning shortcuts, along with formal verification procedures for assessing their presence. 
More recently, Bortolotti et al. \shortcite{bortolottiShortcutsIdentifiabilityConceptbased2025} formalized ``joint reasoning shortcuts'' (JRSs) in CBMs with learned inference layers, showing that reducing deterministic JRSs to zero provably prevents all JRSs and yields identifiability.
Their identifiability result for CBMs with learned inference layers complements ours: they show \emph{when} eliminating deterministic shortcuts suffices, while we characterize \emph{whether} a given constraint set admits shortcuts and at what computational cost.

\subsection{Neurosymbolic Integration and Constraint-Based Learning}
Several frameworks integrate neural perception with symbolic constraints. 
DeepProbLog \cite{manhaeveNeuralProbabilisticLogic2021a} combines neural predicates with probabilistic logic programming; 
Logic Tensor Networks \cite{badreddineLogicTensorNetworks2022} embed logical formulas as differentiable tensor operations; and 
Semantic Loss \cite{DBLP:conf/icml/XuZFLB18} derives differentiable losses from logical constraints via weighted model counting, with recent extensions to autoregressive models \cite{ahmedPseudoSemanticLossAutoregressive}. 
PiShield \cite{stoianPiShieldPyTorchPackage2024} provides shield layers guaranteeing constraint satisfaction during inference regardless of input.
Concept Bottleneck Models (CBMs) \cite{kohConceptBottleneckModels2020} take a complementary approach, routing predictions through human-interpretable concept layers to enable test-time interventions. 
Recent extensions include generative CBMs \cite{ismailConceptBottleneckGenerative2023} and analysis of joint reasoning shortcuts in CBMs with learned inference layers \cite{bortolottiShortcutsIdentifiabilityConceptbased2025}.
These approaches share an implicit assumption that constraints or bottleneck architectures should suffice to ensure the intended concept-label correspondence.
Our work shows this assumption requires formal verification, i.e., constraints may admit multiple valid mappings even when all are satisfied.
More broadly, constraint-based formulations have proven valuable for ML interpretability beyond the neurosymbolic setting.
Ignatiev et al. \shortcite{ignatievReasoningBasedLearningInterpretable2021} survey reasoning and constraint-based approaches to learning inherently interpretable models such as decision trees and decision sets, while Shrotri et al. \shortcite{shrotriConstraintDrivenExplanations2022} use Boolean constraints to guide explanation generation for black-box models. 

In formal explainable AI research, constraint solvers have been used to derive provably correct explanations of black-box predictions \cite{darwicheLogicExplainableAI2023,marquessilvaDeliveringTrustworthyAI2022}.
These approaches share our use of constraint-based reasoning over learned models but address a different question: given a fixed model and input, what is a (minimal) sufficient reason for the prediction? 
Our setting reverses this perspective: given a constraint set used \emph{during} learning, does it admit a unique concept-label correspondence?
Both perspectives concern when symbolic structure provides sufficient determination, but at different stages of the ML pipeline (explanation vs.\ pre-deployment verification).
Our CSP-based analysis of reasoning shortcuts extends this line of work to the neurosymbolic setting, where the goal is not to explain predictions but to verify whether constraints uniquely determine concept-label correspondences.

\subsection{ASP-Based Neurosymbolic Systems and Symmetry Breaking}
ASP has emerged as a tool for integrating symbolic reasoning with neural perception. 
NeurASP \cite{yangNeurASPEmbracingNeural2020} treats neural outputs as probability distributions over ASP atoms, enabling symbolic rules to guide training. 
SLASH \cite{skryaginNeuralProbabilisticAnswerSet2022} scales this via pruning of stochastically insignificant groundings, and Answer Set Networks \cite{skryaginAnswerSetNetworks2024} translate ASP programs into graph neural networks for GPU-accelerated reasoning. 
These systems use ASP for \emph{inference-time probabilistic reasoning}.
Our use of ASP is different: we perform \emph{compile-time verification} of structural properties, checking whether constraints uniquely determine concept mappings before deployment. 
This connects to classical symmetry breaking in CSPs \cite{gent2006symmetry}, but differs in that we must verify a specific ground-truth mapping $\phi^*$ is uniquely determined, not merely find any solution efficiently.

Existing work either detects shortcuts empirically \cite{marconatoNotAllNeuroSymbolic2023,marconatoBEARSMakeNeuroSymbolic2024a,bortolottiNeurosymbolicBenchmarkSuite2024}, quantifies shortcut risk under distributional assumptions \cite{yangAnalysisAbductiveLearning2024}, or assumes constraints suffice for correctness \cite{DBLP:conf/icml/XuZFLB18,kohConceptBottleneckModels2020}. 
We provide the formal foundations: a CSP formalization with necessary conditions (Theorem~\ref{thm:discrimination_necessary}), complexity classification (Theorems~\ref{thm:complexity_conp}--\ref{thm:complexity_repair}), ASP-based verification with correctness guarantees (Algorithm~\ref{alg:asp_verification}), and sample complexity bounds connecting shortcut multiplicity to disambiguation requirements (Theorem~\ref{thm:sample_complexity}). 
This shifts the question from ``how do we detect shortcuts?'' to ``when do constraints guarantee uniqueness?''

%%%%%%%%%%%%%%%%%%%%%%%%%%%%%%%%%%%%%%%%%%%%%%%%%%%%%%%%%%%%%%%%%%%%%%%%%%%%%%%%

\section{Preliminaries}\label{sec:preliminaries}

\begin{definition}[Constraint-Based Neurosymbolic Learning Problem]
\label{def:nesy_problem}
A \emph{constraint-based neurosymbolic learning} (NSL) problem is a tuple $\mathcal{P} = (N, S, C, \phi^*, D)$ where:
\begin{itemize}
    \item $N = \{n_1, \ldots, n_m\}$ is a finite set of $m$ neural outputs
    \item $S = \{s_1, \ldots, s_r\}$ is a finite set of $r$ concept labels
    \item $C$ is a finite set of constraints over mappings $\phi: N \to S$
    \item $\phi^*: N \to S$ is the intended (ground-truth) concept mapping
    \item $D$ is a dataset of observations
\end{itemize}
The triple $(N, S, C)$ defines a constraint satisfaction problem (CSP) whose variables are the neural outputs in $N$, domain is $S$, and constraints are $C$.
\end{definition}

This definition captures NeSy systems in which logical constraints over concept mappings are explicit. 
Frameworks such as DeepProbLog, NeurASP, and Logic Tensor Networks with declared inference rules can be analyzed in our framework by extracting their constraint structure. 
Frameworks that operate without explicit symbolic constraints (e.g., end-to-end learned representations) fall outside our scope.

In practice, the intended mapping $\phi^*$ is unknown to the learner, and recovering it is the goal of neurosymbolic learning.
We include $\phi^*$ in the problem definition because our aim is not to solve the learning task itself, but to \emph{analyze} the constraint structure: given a hypothetical intended mapping, we ask whether the constraints uniquely determine the mapping or admit alternative valid mappings (reasoning shortcuts).
This is analogous to verifying identifiability conditions before deployment, rather than during training.

\begin{definition}[Constraint Graph]
\label{def:constraint_graph}
The \emph{constraint graph} of the CSP $(N, S, C)$ is the undirected graph $G_C=(N,E)$ whose vertex set $N$ is the set of neural outputs and whose edge set is $E = \{(n_i, n_j) \mid n_i \neq n_j \text{ and both } n_i, n_j \text{ appear in some constraint } c \in C\}$.
\end{definition}
Constraint graphs for the running examples are shown in Figure~\ref{fig:constraint_graphs}.

\begin{example}[MNIST-Half as NeSy Learning Problem]
\label{ex:mnist_half_nesy_problem}
The MNIST-Half task instantiates Definition~\ref{def:nesy_problem} as follows:
\begin{itemize}
    \item $N = \{n_0, n_1, n_2, n_3, n_4\}$ (neural outputs for digits 0-4)
    \item $S = \{0, 1, 2, 3, 4\}$ (digit concept labels)
    \item $C = \{C_1, C_2, C_3, C_4\}$ where:
    \begin{align*}
    C_1&: \phi(n_0) + \phi(n_0) = 0 & C_2&: \phi(n_0) + \phi(n_1) = 1\\
    C_3&: \phi(n_2) + \phi(n_3) = 5 & C_4&: \phi(n_2) + \phi(n_4) = 6
    \end{align*}
    \item $\phi^* = (0,1,2,3,4)$
    \item $D$: Dataset of digit pairs with sum labels
\end{itemize}

The constraint graph $G_C$ has two components: $\{n_0, n_1\}$ (edge from $C_2$) and $\{n_2, n_3, n_4\}$ (edges from $C_3, C_4$); the unary constraint $C_1$ pins $\phi(n_0) = 0$ directly.
\end{example}

\begin{definition}[Valid Mapping and Shortcut]
\label{def:valid_shortcut}
Let $\mathcal{P} = (N, S, C, \phi^*, D)$ be a constraint-based NSL problem.

The set of \emph{valid mappings} is:
$$\Phi_C = \{\phi : N \to S \mid \phi \text{ satisfies all constraints in } C\}$$

A valid mapping $\phi \in \Phi_C$ is a \emph{reasoning shortcut} if $\phi \neq \phi^*$.

The problem is \emph{shortcut-free} if $|\Phi_C| = 1$ and the unique element of $\Phi_C$ is $\phi^*$.
\end{definition}

\emph{Standing assumption:} Throughout this paper, we assume the intended mapping satisfies the constraints, i.e., $\phi^* \in \Phi_C$. 
This implies $\Phi_C \neq \varnothing$. 
When this assumption does not hold, we treat it as an error condition (see Algorithm~\ref{alg:asp_verification}, line 9: \texttt{INTENDED-INVALID}).

\begin{remark}[Restriction to Bijections]
In this work, we assume the intended mapping $\phi^*$ is bijective, which is natural for classification tasks where each neural output should correspond to exactly one concept and each concept should be represented exactly once.
We then restrict the search space to bijective mappings when analyzing shortcuts.
This restriction serves as a structural prior: if $\phi^*$ is the unique valid mapping (i.e., $|\Phi_C| = 1$), then non-bijective alternatives are already excluded, so no generality is lost for the shortcut-freeness question.
When $|\Phi_C| > 1$, restricting to bijections may exclude some non-bijective shortcuts, but this is intentional: our analysis focuses on the harder case where shortcuts preserve the bijectivity structure of $\phi^*$.
Algorithm~\ref{alg:asp_verification} supports both bijective and non-bijective enumeration, and our experiments report both (Table~\ref{tab:results}).
\end{remark}

\begin{definition}[Shortcut Multiplicity]
\label{def:shortcut_multiplicity}
The \emph{shortcut multiplicity} of constraint set $C$ is $SM(C) = |\Phi_C| - 1$. 
We denote $k = SM(C)$ throughout. 
An NSL problem is \emph{shortcut-free} if $SM(C) = 0$ (equivalently, $k = 0$).
\end{definition}

\emph{Notation for bijective mappings.}
To avoid ambiguity between general mappings and bijective mappings, we introduce explicit notation:

\begin{itemize}
    \item $\Phi_C^{\mathrm{all}} = \{\phi : N \to S \mid \phi \text{ satisfies all constraints in } C\}$
    \item $\Phi_C^{\mathrm{bij}} = \{\phi \in \Phi_C^{\mathrm{all}} \mid \phi \text{ is bijective}\}$
    \item $SM^{\mathrm{bij}}(C) = |\Phi_C^{\mathrm{bij}}| - 1$
\end{itemize}

Unless stated otherwise, when we write ``$\Phi_C$'' we mean $\Phi_C^{\mathrm{bij}}$ (bijective mappings only), and similarly ``$SM(C)$'' means $SM^{\mathrm{bij}}(C)$. 
For example, Algorithm~\ref{alg:asp_verification} explicitly enforces bijectivity through constraints (lines 4-5).

Wang et al. \shortcite{wangLearningLatentModels2023} study the learnability of multi-instance partial-label learning under a transition function $\sigma$ that maps a tuple of gold labels to a weak supervision signal. 
Their \emph{M-unambiguity} property characterizes when the gold labels can be uniquely recovered from the weak signal. 
Our setting is complementary: rather than inferring unobserved labels from weak supervision, we ask whether a set of logical constraints uniquely determines a concept mapping. 
The following measure adapts their notion to our CSP setting and clarifies the structural relationship between our shortcut multiplicity $SM(C)$, the per-pair disagreement $A(C)$, and the set of disagreement positions $\Delta_C$.

\begin{definition}[M-Ambiguity]\label{def:m_ambiguity}
Adapting the notion of M-unambiguity from \cite{wangLearningLatentModels2023} to our setting, we define the \emph{ambiguity} of a constraint set $C$ as:
$$A(C) = \max_{\phi, \phi' \in \Phi_C} |\{n \in N : \phi(n) \neq \phi'(n)\}|$$
This measures the maximum number of neural outputs on which any two valid mappings disagree.
\end{definition}

\begin{proposition}[Relationship Between Measures]
\label{prop:measures_relationship}
Define the \emph{disagreement set} of constraint set $C$ as $$\Delta_C = \{n \in N : \exists \phi, \phi' \in \Phi_C \text{ with } \phi(n) \neq \phi'(n)\},$$ i.e., the set of positions where valid mappings disagree. Then:
\begin{enumerate}
    \item $SM(C) = 0 \Leftrightarrow |\Phi_C| = 1$
    \item If $SM(C) = 0$, then $A(C) = 0$ and $|\Delta_C| = 0$
    \item If $A(C) = 0$, then $SM(C) = 0$
    \item If $SM(C) > 0$, then $A(C) \geq 1$ and $|\Delta_C| \geq 1$
    \item $A(C) \leq |\Delta_C| \leq |N|$ (with equality $|\Delta_C| = |N|$ only when 
          mappings disagree everywhere)
\end{enumerate}
\end{proposition}

\begin{proof}
Parts (1)-(4) follow directly from definitions.
For part (5): By definition, $A(C)$ measures the maximum disagreement 
between any single pair of mappings, while $\Delta_C$ is the union over all pairs. 
The maximum pairwise disagreement cannot exceed the union: $A(C) \leq |\Delta_C|$.
Furthermore, by definition $\Delta_C \subseteq N$, so $|\Delta_C| \leq |N|$. 
Equality $|\Delta_C| = |N|$ occurs when for every position $n \in N$, there exist $\phi, \phi' \in \Phi_C$ with $\phi(n) \neq \phi'(n)$. 
\end{proof}

\section{Theoretical Framework}\label{sec:discrimination}

We now investigate under what conditions a constraint set uniquely determines the intended mapping. 
We begin by identifying a necessary condition, the discrimination property, then show its limitations, examine the symmetry structure of the solution space, analyze the sample complexity of disambiguation when constraints alone are insufficient, and describe a practical label selection strategy with a formal bound.

\subsection{Discrimination is Necessary}\label{sec:necessary_conditions}

\begin{definition}[Discrimination Property]
\label{def:discrimination}
Let $\phi: N \to S$ be a concept mapping. 
For distinct concepts $s_i, s_j \in S$, define the \emph{transposed mapping} $\phi_{s_i \leftrightarrow s_j}: N \to S$ by:
$$\phi_{s_i \leftrightarrow s_j}(n) = \begin{cases}
s_j & \text{if } \phi(n) = s_i \\
s_i & \text{if } \phi(n) = s_j \\
\phi(n) & \text{otherwise}
\end{cases}$$

A constraint set $C$ is \emph{discriminative} if:
$$\forall s_i, s_j \in S \text{ with } s_i \neq s_j, \; \forall \phi \in \Phi_C: \quad \phi_{s_i \leftrightarrow s_j} \notin \Phi_C$$
\end{definition}
This definition quantifies over all \(\phi \in \Phi_C\) for clarity. 
Because shortcut-freeness implies there is only one solution, i.e., \(\Phi_C=\{\phi^*\}\), under shortcut-freeness this reduces to discrimination at \(\phi^*\).
Informally, discrimination forbids all 2-cycles (transpositions) in the automorphism group of valid mappings (see Definition~\ref{def:automorphism} and Proposition~\ref{prop:rigidity}).

\begin{theorem}[Discrimination is necessary]
\label{thm:discrimination_necessary}
Let $\mathcal{P} = (N, S, C, \phi^*, D)$ with $|N| = |S|$. 
If $\mathcal{P}$ is shortcut-free, then $C$ is discriminative.
\end{theorem}

\begin{proof}
Assume $\mathcal{P}$ is shortcut-free, so $\Phi_C = \{\phi^*\}$.
Consider any distinct $s_i, s_j \in S$. 
Since $\phi^*$ is bijective, there exist distinct $n_i, n_j \in N$ with $\phi^*(n_i) = s_i$ and $\phi^*(n_j) = s_j$.
The transposition swaps these assignments, so $\phi^*_{s_i \leftrightarrow s_j} \neq \phi^*$. 
Since $|\Phi_C| = 1$, we conclude $\phi^*_{s_i \leftrightarrow s_j} \notin \Phi_C$.
As this holds for all distinct pairs, $C$ is discriminative.
\end{proof}

\begin{example}[Discrimination in MNIST-Half]
\label{ex:mnist_half_discrimination}
In MNIST-Half (Example~\ref{ex:mnist_half_nesy_problem}), consider the intended mapping $\phi^* = (0, 1, 2, 3, 4)$.

Attempt the transposition $2 \leftrightarrow 3$:
$$\phi^*_{2 \leftrightarrow 3} = (0, 1, \mathbf{3}, \mathbf{2}, 4)$$

Check constraints:
\begin{itemize}
    \item $C_3$: $\phi(n_2) + \phi(n_3) = 3 + 2 = 5$ (satisfied)
    \item $C_4$: $\phi(n_2) + \phi(n_4) = 3 + 4 = 7 \neq 6$ (violated)
\end{itemize}

The transposition \emph{violates} $C_4$, so $\phi^*_{2 \leftrightarrow 3} \notin \Phi_C$. 
Thus $C$ discriminates the pair $(2,3)$ at $\phi^*$.
\end{example}

\subsection{Discrimination Can Fail in Practice}\label{sec:discrimination_fail}

Theorem~\ref{thm:discrimination_necessary} shows discrimination is necessary, but the following example shows that it is not sufficient.

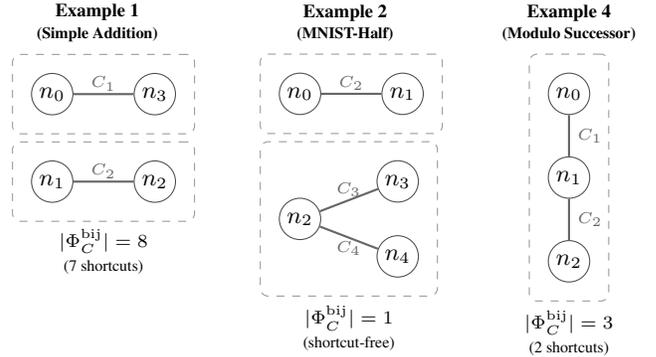
\begin{figure}[htb]
\centering
\begin{tikzpicture}[
    node distance=0.7cm and 0.8cm,
    nnode/.style={circle, draw=black!70, fill=none, minimum size=0.55cm, 
                  inner sep=0pt, font=\footnotesize},
    edge/.style={thick, black!60},
    clabel/.style={font=\tiny, fill=none, inner sep=1pt},
    titlestyle/.style={font=\scriptsize\bfseries, align=center},
    compbox/.style={draw=black!40, dashed, rounded corners=3pt, inner sep=7pt},
    resultlabel/.style={font=\scriptsize, align=center},
]

% ============================================
% LEFT: 4-node example (Example 1)
% ============================================
\node[titlestyle] (title1) {Example~\ref{ex:simple_addition_intro}\\[-2pt]{\fontsize{6}{7}\selectfont (Simple Addition)}};

\node[nnode, below=0.3cm of title1, xshift=-0.6cm] (a0) {$n_0$};
\node[nnode, right=of a0] (a3) {$n_3$};
\node[nnode, below=0.6cm of a0] (a1) {$n_1$};
\node[nnode, right=of a1] (a2) {$n_2$};

% Edges
\draw[edge] (a0) -- node[clabel, above] {$C_1$} (a3);
\draw[edge] (a1) -- node[clabel, above] {$C_2$} (a2);

% Component grouping (dashed boxes)
\node[compbox, fit=(a0)(a3)] {};
\node[compbox, fit=(a1)(a2)] {};

% Result label
\node[resultlabel, below=0.5cm of $(a1)!0.5!(a2)$] (res1) {%
    $|\Phi_C^{\mathrm{bij}}| = 8$\\[-1pt]
    {\tiny (7 shortcuts)}};

% ============================================
% MIDDLE: MNIST-Half (Example 2)
% ============================================
\node[titlestyle, right=1.6cm of title1] (title2) {Example~\ref{ex:mnist_half_nesy_problem}\\[-2pt]{\fontsize{6}{7}\selectfont (MNIST-Half)}};

% Component 1: n0 -- n1 via C2
\node[nnode, below=0.3cm of title2, xshift=-0.6cm] (b0) {$n_0$};
\node[nnode, right=of b0] (b1) {$n_1$};

% Component 2: n2 central, connected to n3 and n4
\node[nnode, below=1.1cm of b0] (b2) {$n_2$};
\node[nnode, above right=0.1cm and 0.9cm of b2] (b3) {$n_3$};
\node[nnode, below right=0.1cm and 0.9cm of b2] (b4) {$n_4$};

% Add C1 as a self-loop on n0
% \draw[edge] (b0) edge[loop left, looseness=4] node[clabel, left] {$C_1$} (b0);
% \node[inner sep=0pt, left=0.35cm of b0] (b0loop) {};

% Edges -- component 1
\draw[edge] (b0) -- node[clabel, above] {$C_2$} (b1);

% Edges -- component 2
\draw[edge] (b2) -- node[clabel, above] {$C_3$} (b3);
\draw[edge] (b2) -- node[clabel, below] {$C_4$} (b4);

% Component grouping (dashed boxes)
\node[compbox, fit=(b0)(b1)] {};
\node[compbox, fit=(b2)(b3)(b4)] {};

% Result label
\node[resultlabel, below=0.5cm of $(b2.south)!0.5!(b4.south)$] (res2) {%
    $|\Phi_C^{\mathrm{bij}}| = 1$\\[-1pt]
    {\tiny (shortcut-free)}};

% ============================================
% RIGHT: Example 4 (3-node mod)
% ============================================
\node[titlestyle, right=1.2cm of title2] (title3) {Example~\ref{cex:discrimination_fails}\\[-2pt]{\fontsize{6}{7}\selectfont (Modulo\ Successor)}};

\node[nnode, below=0.3cm of title3] (m0) {$n_0$};
\node[nnode, below=0.55cm of m0] (m1) {$n_1$};
\node[nnode, below=0.55cm of m1] (m2) {$n_2$};

% Edges
\draw[edge] (m0) -- node[clabel, right, xshift=2pt] {$C_1$} (m1);
\draw[edge] (m1) -- node[clabel, right, xshift=2pt] {$C_2$} (m2);

% Single component box (connected)
\node[compbox, fit=(m0)(m1)(m2)] {};

% Result label
\node[resultlabel, below=0.25cm of m2] (res3) {%
    $|\Phi_C^{\mathrm{bij}}| = 3$\\[-1pt]
    {\tiny (2 shortcuts)}};

\end{tikzpicture}
\caption{Constraint graphs for the three running examples.
Example~\ref{ex:simple_addition_intro} has two disconnected components, admitting 7 bijective shortcuts.
Example~\ref{ex:mnist_half_nesy_problem} (MNIST-Half) also has disconnected components, but stronger arithmetic constraints achieve uniqueness under bijectivity (a unary constraint is omitted from the figure).
Example~\ref{cex:discrimination_fails} has a connected graph, yet admits 2 shortcuts via rotational symmetry (3-cycles), showing that connectivity alone does not prevent shortcuts.}
\label{fig:constraint_graphs}
\end{figure}

\begin{example}[Discrimination is necessary but not sufficient]
\label{cex:discrimination_fails}
Consider a 3-node problem with modulo successor constraints:
\begin{itemize}
    \item $N = \{n_0, n_1, n_2\}$
    \item $S = \{0, 1, 2\}$
    \item Constraints: \\
        $C_1: \phi(n_1) \equiv \phi(n_0) + 1 \pmod{3}$; 
        $C_2: \phi(n_2) \equiv \phi(n_1) + 1 \pmod{3}$
    \item Intended: $\phi^* = (0, 1, 2)$
\end{itemize}

The constraint graph is \emph{connected} ($n_0 - n_1 - n_2$), yet three bijective valid mappings exist:
\begin{itemize}
    \item $\phi_1 = (0, 1, 2)$ \quad (intended, $\phi^*$)
    \item $\phi_2 = (1, 2, 0)$ \quad (rotation by 1)
    \item $\phi_3 = (2, 0, 1)$ \quad (rotation by 2)
\end{itemize}
Each $\phi_i$ assigns concepts so that consecutive outputs differ by 1 modulo 3, satisfying both $C_1$ and $C_2$.

\emph{Discrimination holds:}
Consider $\phi_1 = (0,1,2)$. 
Swapping $0 \leftrightarrow 1$ gives $(1,0,2)$: check $C_1$: $0 \equiv 1 + 1 \pmod{3}$? No. 
Similarly, every transposition applied to any valid mapping violates at least one constraint. 
Thus $C$ satisfies the discrimination property.

\emph{Shortcuts exist via 3-cycles:} 
The cyclic permutation $\sigma: 0 \mapsto 1,\; 1 \mapsto 2,\; 2 \mapsto 0$ acts on solutions by composition $(\sigma \circ \phi)(n) = \sigma(\phi(n))$:
\begin{align*}
    \sigma \circ \phi_1 &= (1, 2, 0) = \phi_2,\\
    \sigma \circ \phi_2 &= (2, 0, 1) = \phi_3,\\
    \sigma \circ \phi_3 &= (0, 1, 2) = \phi_1.
\end{align*}
So $\sigma \in \mathrm{Aut}(\mathcal{X})$. 
Together with $\sigma^2$ and the identity, $\mathrm{Aut}(\mathcal{X}) = \{\mathrm{id}, \sigma, \sigma^2\} \cong \mathbb{Z}/3$ is non-trivial, and acts transitively on $\Phi_C$.

Discrimination eliminates 2-cycles from $\mathrm{Aut}(\mathcal{X})$, but cannot detect 3-cycles, or longer cycles in general. 
This shows that discrimination, while necessary (Theorem~\ref{thm:discrimination_necessary}), is not sufficient for shortcut-freeness, even when the constraint graph is connected.
\end{example}

Example~\ref{cex:discrimination_fails} shows that the gap between discrimination and shortcut-freeness is not merely a matter of constraint graph connectivity: the modulo successor constraints form a connected graph, yet shortcuts persist through 3-cycles that discrimination cannot detect (Figure~\ref{fig:constraint_graphs}).
This contrasts with Example~\ref{ex:simple_addition_intro}, where shortcuts arise from a different mechanism, i.e., disconnected components allowing independent transpositions.
Together, these examples show that discrimination is not a practical guardrail on its own, and the question of sufficient conditions for uniqueness leads naturally to analyzing value symmetries.

\subsection{Value-Symmetry Elimination}\label{sec:sufficient_conditions}

Having established the limits of discrimination, we turn to sufficient conditions by examining the symmetry structure of the solution space.

\begin{definition}[Symmetric Group and Automorphism Group]
\label{def:automorphism}
Let $\mathcal{X} = (N, S, C)$ be the CSP instance restricted to bijections $\phi: N \to S$.
A \emph{permutation} of $S$ is a bijection $\sigma: S \to S$.
The set of all permutations of $S$ forms the \emph{symmetric group} $\mathrm{Sym}(S)$.
Each $\sigma \in \mathrm{Sym}(S)$ acts on solutions by composition: $(\sigma \circ \phi)(n) = \sigma(\phi(n))$ for all $n \in N$.

The \emph{automorphism group} of $\mathcal{X}$ is:
$$\mathrm{Aut}(\mathcal{X}) = \{\sigma \in \mathrm{Sym}(S) : \sigma \circ \phi \in \Phi_C \text{ for all } \phi \in \Phi_C\}$$
\end{definition}

\begin{proposition}[Value-Symmetry Elimination]
\label{prop:rigidity}
If $\mathrm{Aut}(\mathcal{X})$ is trivial (i.e., $\mathrm{Aut}(\mathcal{X}) = \{\mathrm{id}\}$), then no two distinct solutions in $\Phi_C$ are related by a permutation in $\mathrm{Aut}(\mathcal{X})$.
\end{proposition}

\begin{proof}
If $\phi' = \sigma \circ \phi$ for some $\sigma \in \text{Aut}(\mathcal{X})$, triviality forces $\sigma = \text{id}$, hence $\phi' = \phi$.
\end{proof}

Proposition~\ref{prop:rigidity} shows that a trivial automorphism group eliminates the multiplicity arising from value-permutation symmetries, but does not imply $|\Phi_C| = 1$: a permutation $\sigma \notin \mathrm{Aut}(\mathcal{X})$ might still map one particular solution to another.
The gap between trivial $\mathrm{Aut}(\mathcal{X})$ and uniqueness arises from \emph{symmetries that preserve some valid mappings but not all}. 
We illustrate two distinct mechanisms by which shortcuts arise, and note that they have different repair implications.

\emph{Single-orbit shortcuts (Example~\ref{cex:discrimination_fails}).}
Here $\mathrm{Aut}(\mathcal{X}) = \{\mathrm{id}, \sigma, \sigma^2\} \cong \mathbb{Z}/3$, where $\sigma$ is the 3-cycle $0 \mapsto 1 \mapsto 2 \mapsto 0$.
The group $\mathrm{Aut}(\mathcal{X})$ acts transitively on $\Phi_C$: every valid mapping is reachable from any other via a single rotation. 
Repair is correspondingly cheap: pinning a single position $\phi(n^*) = \phi^*(n^*)$ breaks the orbit and identifies $\phi^*$ uniquely (one constraint suffices, see Theorem~\ref{thm:repair_correctness}).

\emph{Component-decomposable shortcuts (Example~\ref{ex:simple_addition_intro}).}
Here disconnected components of the constraint graph admit independent rearrangements. 
The two components $\{n_0, n_3\}$ and $\{n_1, n_2\}$ each support an independent transposition, yielding $4 \times 2 = 8$ valid bijections. 
Repair requires \emph{at least one constraint per component admitting non-trivial rearrangement}, since pinning one component leaves the other free.

This contrast motivates Section~\ref{sec:algorithms}: verification via Algorithm~\ref{alg:asp_verification} detects shortcuts regardless of
mechanism, while greedy repair via Algorithm~\ref{alg:repair} resolves them at a cost determined jointly by the orbit structure of $\mathrm{Aut}(\mathcal{X})$ and the connectivity of $G_C$.

Discrimination (Definition~\ref{def:discrimination}) eliminates all 2-cycles from the value-symmetry group, but does not eliminate longer cycles or non-symmetry-based multiplicity. Two structural properties appear relevant but remain unproven as sufficient conditions.
First, \emph{constraint graph connectivity} (Definition~\ref{def:constraint_graph}) prevents independent component rearrangements as in Example~\ref{ex:simple_addition_intro}, but connected graphs may still admit shortcuts through longer permutation cycles (Example~\ref{cex:discrimination_fails}; see also Figure~\ref{fig:constraint_graphs}).
Second, \emph{constraint strength}: in MNIST-Half, arithmetic constraints like $\phi(n_2) + \phi(n_4) = 6$ fully determine values when combined with bijectivity, while the modulo successor constraints in Example~\ref{cex:discrimination_fails} do not. 
Formalizing when constraints are ``strong enough'' and whether connectivity and strength jointly suffice for uniqueness remains an open problem.

\subsection{Sample Complexity for Disambiguation}\label{sec:sample_complexity}

Since complete sufficient conditions for shortcut-freeness remain open, a natural fallback is \emph{disambiguation through labeled data}.
Given that $|\Phi_C| > 1$, how many concept labels are necessary to identify $\phi^*$ among the valid mappings? 
This question is of practical importance because even when constraints alone are insufficient, a small number of strategically chosen labels may resolve the remaining ambiguity. 
We formalize this as a query complexity problem over the candidate set $\Phi_C$.

\begin{theorem}[Sample Complexity of Label Queries]
\label{thm:sample_complexity}
Let $\Phi_C$ contain $k+1$ valid bijective mappings with $SM(C) = k \geq 1$.
Let $\Delta_C = \{n \in N : \exists \phi, \phi' \in \Phi_C,\; \phi(n) \neq \phi'(n)\}$ and let $r = |S|$.
Suppose a learner can make \emph{label queries}: given $n \in N$, an oracle returns $\phi^*(n) \in S$.
The learner may choose each query adaptively based on previous answers.
Then we have: 
\begin{enumerate}
    \item \emph{Lower bound:} Any learner requires at least $\lceil \log_r(k+1) \rceil$ queries to identify $\phi^*$.
    \item \emph{Upper bound:} Querying all positions in $\Delta_C$ identifies $\phi^*$ using at most $|\Delta_C| \leq r$ queries.
\end{enumerate}
The lower bound is achieved when $\Phi_C$ admits balanced partitions at each query step; the upper bound is achieved when each query eliminates only one candidate.
\end{theorem}

\begin{proof}
\emph{Lower bound:}
Each label query returns a value in $S$, providing at most $\log_2 |S|$ bits of information. 
Distinguishing among $k+1$ candidates requires at least $\log_2(k+1)$ bits. 
Therefore at least $\lceil \log_2(k+1) / \log_2 |S| \rceil = \lceil \log_{|S|}(k+1) \rceil$ queries are necessary, regardless of the adaptive strategy.

\emph{Upper bound:}
Query $\phi^*(n)$ for each $n \in \Delta_C$. 
After all queries, consider any remaining candidate $\phi \in \Phi_C$:
\begin{itemize}
    \item For $n \in \Delta_C$: the query at $n$ eliminates all $\phi$ with $\phi(n) \neq \phi^*(n)$, so $\phi(n) = \phi^*(n)$.
    \item For $n \in N \setminus \Delta_C$: by definition of $\Delta_C$, all valid mappings agree at $n$, so $\phi(n) = \phi^*(n)$.
\end{itemize}
Thus $\phi = \phi^*$, and $|\Delta_C|$ queries suffice.
Note that $\phi^*$ is never eliminated since it always agrees with its own labels.

\emph{Bound on $|\Delta_C|$:}
The $k+1$ valid mappings are pairwise distinct when restricted to $\Delta_C$ (they agree on $N \setminus \Delta_C$), so $k + 1 \leq r^{|\Delta_C|}$, giving $|\Delta_C| \geq \lceil \log_r(k+1) \rceil$.
\end{proof}

The gap between the lower and upper bounds depends on the structure of $\Phi_C$: favorable structure (as in Example~\ref{ex:sample_complexity_demo} below) allows the lower bound to be achieved, while adversarial cases may require all $|\Delta_C|$ queries.
In practice, several strategies approximate the oracle:
\begin{itemize}
    \item \emph{Random sampling} requires $O(|N|/|\Delta_C| \cdot k \log_2 k)$ expected queries
    \item \emph{Uncertainty sampling} achieves $O(\min\{|\Delta_C|, k\})$ queries (Proposition~\ref{prop:uncertainty_bound} in Section~\ref{sec:practical_guide})
    \item \emph{Greedy disambiguation} approximates the lower bound
\end{itemize}
These practical strategies are discussed in Section~\ref{sec:practical_guide}, with additional details in the supplementary material.

\begin{example}[Sample Complexity Demonstration]
\label{ex:sample_complexity_demo}
For the 4-node addition problem (Example~\ref{ex:simple_addition_intro}) with $k = 7$, $|S| = 4$, and $\Delta_C = \{n_0, n_1, n_2, n_3\}$:

\emph{Bounds:} Lower $= \lceil \log_4(8) \rceil = 2$; \quad Upper $= |\Delta_C| = 4$.

\emph{Achieved in 2 label queries} (matching the lower bound):
\begin{itemize}
    \item \textit{Query 1:} ``What is $\phi^*(n_0)$?'' Oracle returns $0$.
    Among the 8 valid bijections, only those with $\phi(n_0) = 0$ remain.
    Direct enumeration gives exactly two:
    $\phi_1 = (0, 1, 2, 3)$ and $\phi_2 = (0, 2, 1, 3)$. 
    The first query eliminates 6 of 8 candidates because the constraint $\phi(n_0) + \phi(n_3) = 3$, combined with bijectivity, forces $\phi(n_3) = 3$ once $\phi(n_0) = 0$.
    \item \textit{Query 2:} ``What is $\phi^*(n_1)$?'' Oracle returns $1$.
    This eliminates $\phi_2$ (which has $\phi_2(n_1) = 2$).
    Only $\phi_1 = \phi^*$ remains. Identification complete.
\end{itemize}
In adversarial cases (e.g., when each query removes only one candidate), all $|\Delta_C| = 4$ queries would be needed, matching the upper bound.
\end{example}

\subsection{Practical Label Selection Strategies}
\label{sec:practical_guide}

While Theorem~\ref{thm:sample_complexity} provides information-theoretic bounds assuming an oracle, we now present a practical strategy with a formal bound.
The strategy assumes access to (or maintenance of) the candidate set $\Psi \subseteq \Phi_C$, which can be approximated via bounded model enumeration (e.g., \texttt{clingo --models=100}), diversified sampling, or incremental verification.

\emph{Uncertainty Sampling.}
Maintain candidate valid mappings $\Psi \subseteq \Phi_C$ (discovered via Algorithm~\ref{alg:asp_verification}). 
At each step:
\begin{enumerate}
    \item For each $n \in N$, compute disagreement: $d(n) = |\{s \in S : \exists \phi \in \Psi, \phi(n) = s\}|$
    \item Select $n^* = \arg\max_{n \in N} d(n)$ (output with most disagreement)
    \item Query label $(n^*, \phi^*(n^*))$
    \item Remove all $\phi \in \Psi$ with $\phi(n^*) \neq \phi^*(n^*)$
\end{enumerate}

\begin{proposition}[Uncertainty Sampling Bound]
\label{prop:uncertainty_bound}
Assume oracle access to the full candidate set $\Psi=\Phi_C$.
Under uncertainty sampling, at most $\min\{k, |\Delta_C|\}$ labeled examples are needed to identify $\phi^*$, where $\Delta_C$ is the set of disagreement positions.
\end{proposition}

\begin{proof}
Each query selects position $n^*$ with maximum disagreement and obtains label $\phi^*(n^*)$. 
This eliminates all candidates $\phi \in \Psi$ with $\phi(n^*) \neq \phi^*(n^*)$, which is at least one candidate (since $d(n^*) \geq 2$ when $|\Psi| \geq 2$).
Once we query position $n$, all remaining candidates agree with $\phi^*$ at $n$, so each query uses a distinct position.
In the best case, each query eliminates approximately half the remaining candidates, requiring $\lceil \log_2(k+1) \rceil$ queries.
In the worst case, each query eliminates exactly one candidate, requiring $k$ queries but bounded by the number of disagreement positions: $\min(k, |\Delta_C|)$.
\end{proof}

This is an information-theoretic bound assuming perfect knowledge of $\Psi$, not a computational guarantee.

\section{Algorithms and Complexity}\label{sec:algorithms}

\subsection{Algorithmic Verification}\label{sec:verification}

We now translate the theoretical framework into a practical verification algorithm using Answer Set Programming (ASP).
Our use of ASP is not for solving the task instance, but for \textit{reasoning} about the constraint theory itself: verifying uniqueness, counting alternatives, and guiding repair.

\begin{algorithm}
\caption{Shortcut Verification (Pseudocode)}
\label{alg:asp_verification}
\begin{algorithmic}[1]
\Require Constraint set $C$, concepts $S$, neural outputs $N$ with $|N| = |S|$, intended mapping $\phi^*$
\Ensure \texttt{SHORTCUT-FREE}, \texttt{INTENDED-INVALID}, or list of shortcuts

\State // Initialize verification problem
\State Initialize search space: all mappings $\phi: N \to S$
\State Add all constraints from $C$
\State Add bijectivity constraint: 
\State \qquad each output (resp. concept) maps to exactly one concept (resp. output)

\State // Step 1: Verify intended mapping
\For{each constraint $c \in C$}
    \If{$\phi^*$ violates $c$}
        \State
        \Return \texttt{INTENDED-INVALID}
    \EndIf
\EndFor

\State // Step 2: Search for alternative valid mappings
\State Exclude $\phi^*$ from search space
\State $\Psi \leftarrow$ Find all valid bijective mappings

\If{$\Psi = \emptyset$}
    \State
    \Return \texttt{SHORTCUT-FREE}
\Else
    \State
    \Return $\Psi$ (each $\phi \in \Psi$ is a shortcut)
\EndIf
\end{algorithmic}
\end{algorithm}

The ASP encoding follows a uniform pattern.
For a given constraint-based neurosymbolic learning problem $\mathcal{P} = (N, S, C, \phi^*, D)$:
(i) declare neural outputs and concepts as domain atoms,
(ii) use choice rules with cardinality constraints to generate candidate mappings, optionally enforcing bijectivity via cardinality constraints (one per neural output and one per concept), 
(iii) encode each task constraint as an integrity constraint over value atoms, and 
(iv) exclude \(\phi^*\) via a joint integrity constraint to enumerate only alternative valid mappings. 
For instance, the arithmetic constraint \(\phi(n_0) + \phi(n_3) = 3\) becomes \texttt{:- value(n0,V0), value(n3,V3), V0+V3 != 3}.
This encoding is directly executable in \textit{clingo} with \texttt{--models=0} for full enumeration. 
Complete ASP encodings for all examples are available in the supplementary material.

\begin{proposition}[Algorithm Correctness]
\label{prop:asp_correctness}
Algorithm~\ref{alg:asp_verification} is correct:
\begin{enumerate}
    \item If it returns \texttt{INTENDED-INVALID}, then $\phi^* \notin \Phi_C$
    \item If it returns \texttt{SHORTCUT-FREE}, then $\phi^* \in \Phi_C$ and no other bijective mapping satisfies $C$
    \item If it returns shortcuts, then $\phi^* \in \Phi_C$ and there exist other bijective mappings satisfying $C$
\end{enumerate}
\end{proposition}

\begin{proof}
\emph{Part 1:} Lines 7-9 explicitly check if $\phi^*$ satisfies all constraints in $C$. 
If any constraint is violated, the algorithm returns \texttt{INTENDED-INVALID} (line 9), so $\phi^* \notin \Phi_C$.

\emph{Part 2 (Soundness):} 
If the algorithm returns \texttt{SHORTCUT-FREE}, then $\phi^*$ passed all constraint checks (lines 7-9), so $\phi^* \in \Phi_C$.
After excluding $\phi^*$ (line 11), the search space consists of all bijective mappings satisfying $C$ except $\phi^*$ (lines 2-5, 12).
Since $\Psi = \emptyset$ (line 13), no other bijective mapping satisfies $C$.

\emph{Part 3 (Completeness):} 
If $\phi^* \in \Phi_C$ and other bijective mappings satisfy $C$, then Step~1 succeeds, and any bijective $\phi' \neq \phi^*$ satisfying $C$ is found and included in $\Psi$ (line 12), which is returned (line 16).
\end{proof}

%%%%%%%%%%%%%%%%%%%%%%%%%%%%%%%%%%%%%%%%%%%%%%%%%%%%%%%%%%%%%%%%%%%%%%%%%%%%%%%%%%%%

\subsection{Constraint Repair}\label{sec:constraint_repair}

When verification reveals shortcuts, the next step is to eliminate them by augmenting the constraint set.
Algorithm~\ref{alg:repair} adds \emph{unary pinning constraints} of the form $\phi(n) = \phi^*(n)$ for some $n \in N$. 
The NP-hardness result in Theorem~\ref{thm:complexity_repair} is stated for the more general decision problem in which candidate constraints are drawn from a fixed library of pinning constraints, where each candidate may pin one or more positions.

\begin{algorithm}
\caption{Greedy Shortcut Repair}
\label{alg:repair}
\begin{algorithmic}[1]
\Require Constraint set $C$, intended mapping $\phi^*$, maximum iterations $T$
\Require $\phi^* \in \Phi_C$ (intended mapping must be initially valid)
\Ensure Either \texttt{SHORTCUT-FREE} constraint set $C'$ or \texttt{TIMEOUT} with current $C'$

\State $C' \leftarrow C$, $t \leftarrow 0$
\Repeat
    \State $\Psi \leftarrow$ Run Algorithm~\ref{alg:asp_verification} with $(C', \phi^*)$
    \If{$\Psi = $ \texttt{SHORTCUT-FREE}}
        \State
        \Return $C'$
    \ElsIf{$\Psi = $ \texttt{INTENDED-INVALID}}
        \State
        \Return ERROR: $\phi^*$ inconsistent with constraints
    \EndIf
    \State Select $\phi_{sc} \in \Psi$ (any detected shortcut)
    \State $D_{sc} \leftarrow \{n \in N : \phi^*(n) \neq \phi_{sc}(n)\}$
    \State Select $n^* \in D_{sc}$ (arbitrary or by heuristic)
    \State $c_{new} \leftarrow $ constraint ``$\phi(n^*) = \phi^*(n^*)$''
    \State $C' \leftarrow C' \cup \{c_{new}\}$
    \State $t \leftarrow t + 1$
\Until{$t \geq T$}\\
\Return \texttt{TIMEOUT} with current $C'$
\end{algorithmic}
\end{algorithm}

\begin{theorem}[Repair Algorithm Correctness]
\label{thm:repair_correctness}
Let $C_0 = C$ and $C_i$ be the constraint set after iteration $i$ of Algorithm~\ref{alg:repair}. 
Assume $\phi^* \in \Phi_{C_0}$, $|\Phi_{C_0}| < \infty$ (finite number of bijections from $N$ to $S$ with $|N| = |S|$), and $T \geq k$ where $k = SM(C_0)$. 
Then:
\begin{enumerate}
    \item \emph{Monotonic decrease}: $SM(C_{i+1}) < SM(C_i)$ for each iteration that adds a constraint
    \item \emph{Preserves intended}: $\phi^* \in \Phi_{C_i}$ for all $i$
    \item \emph{Termination bound}: Algorithm terminates in at most $k$ iterations with $SM(C') = 0$
    \item \emph{Suboptimality}: The algorithm does not minimize $|C' \setminus C|$
\end{enumerate}
\end{theorem}

\begin{proof}
\emph{(1) Monotonic decrease:}

At iteration $i$, let $\phi_{sc} \in \Phi_{C_i}$ be the detected shortcut. The algorithm adds constraint $c_{new}: \phi(n^*) = \phi^*(n^*)$ where $\phi^*(n^*) \neq \phi_{sc}(n^*)$.

Clearly $\phi_{sc}$ violates $c_{new}$, so $\phi_{sc} \notin \Phi_{C_{i+1}}$.

Thus $\Phi_{C_{i+1}} \subset \Phi_{C_i}$ (strict subset), hence $|\Phi_{C_{i+1}}| < |\Phi_{C_i}|$, so $SM(C_{i+1}) < SM(C_i)$.

\emph{(2) Preserves intended:}

The constraint $c_{new}: \phi(n^*) = \phi^*(n^*)$ is satisfied by $\phi^*$ by construction. Thus if $\phi^* \in \Phi_{C_i}$, then $\phi^* \in \Phi_{C_{i+1}}$.

By induction: $\phi^* \in \Phi_{C_0}$ by assumption, hence $\phi^* \in \Phi_{C_i}$ for all $i$.

\emph{(3) Termination:}

Each iteration strictly decreases $|\Phi_{C_i}|$ by at least 1 (part 1).

Initially: $|\Phi_{C_0}| = SM(C) + 1 = k + 1 < \infty$

After iteration $i$: $|\Phi_{C_i}| \leq k + 1 - i$

To reach $|\Phi_{C_k}| = 1$, we need at most $k$ iterations.

By part (2), $\phi^* \in \Phi_{C_k}$, so $\Phi_{C_k} = \{\phi^*\}$, thus $SM(C_k) = 0$.

Since the timeout assumption $T \geq k$ ensures the algorithm does not exit before iteration $k$, and the decreasing chain terminates (part 1), Algorithm~\ref{alg:repair} returns $C_k$ with $SM(C_k)=0$.

\emph{(4) Suboptimality:}

The algorithm is greedy and therefore does not minimize $|C'\setminus C|$.
Moreover, the more general candidate-library repair problem is NP-hard (Theorem~\ref{thm:complexity_repair}), indicating that optimal repair is computationally difficult in general.
\end{proof}

\subsection{Application to Existing Frameworks}
The verification and repair workflow applies directly to NeSy frameworks such as DeepProbLog, NeurASP, and Logic Tensor Networks.
First, export the logical rules as ASP constraints, define $\phi^*$ and the 
neural output/concept sets with $|N| = |S|$, then run Algorithm~\ref{alg:asp_verification} to check shortcut-freeness. 
If shortcuts are detected, Algorithm~\ref{alg:repair} eliminates them with convergence guarantees, or uncertainty sampling (Section~\ref{sec:practical_guide}) can resolve ambiguity via labeled examples (see supplementary material for a step-by-step guide).

%%%%%%%%%%%%%%%%%%%%%%%%%%%%%%%%%%%%%%%%%%%%%%%%%%%%%%%%%%%%%%%%%%%%%%%%%%%%%%%%%%%%

\subsection{Complexity Results}\label{sec:complexity}

Given a constraint-based NSL problem with intended mapping \(\phi^*\):
(i) deciding if the problem is shortcut-free is coNP-complete, 
(ii) counting shortcuts is \#P-complete, and
(iii) finding minimal repair is NP-hard.

\begin{theorem}[Shortcut-Free verification is coNP-complete]
\label{thm:complexity_conp}
    Given a constraint-based NSL problem $\mathcal{P} = (N, S, C, \phi^*, D)$ with $|N| = |S|$, deciding if $SM(C) = 0$ (given $\phi^*$) is \emph{coNP-complete}.
\end{theorem}

\begin{proof} (Proof sketch)
\emph{Membership in coNP:} a certificate that $SM(C)>0$ is any valid mapping $\phi \neq \phi^*$, verifiable in polynomial time.
\emph{coNP-Hardness:} we reduce UNSAT to $SM(C)=0$.
Given a CNF formula $\psi$ over Boolean variables $\{x_1,\ldots,x_m\}$, construct a constraint-based NSL problem with $2m$ neural outputs (one per literal) and $2m$ concepts (paired truth values $T_i,F_i$).
Add constraints enforcing neural outputs $\{n_i, \bar{n}_i\}$ must map to concepts from $\{T_i, F_i\}$. 
The intended mapping $\phi^*$ encodes the all-true assignment over $\{x_1,\ldots,x_m,y\}$ for an auxiliary formula $\psi' = \psi \wedge \neg y$ with fresh variable $y$, ensuring $\phi^*$ is never a satisfying assignment of $\psi'$.
Constraints use an auxiliary atom $alt$ that activates iff $\phi \neq \phi^*$.
When active, clause satisfaction constraints enforce that the encoded assignment satisfies $\psi$.
By construction, $\phi^* \in \Phi_C$, and any $\phi \neq \phi^*$ is valid iff it encodes a satisfying assignment of $\psi'$. 
Since $\#\mathrm{SAT}(\psi') = \#\mathrm{SAT}(\psi)$ and $\psi'$ is satisfiable iff $\psi$ is, the reduction preserves both UNSAT and counting.
Thus $SM(C)=0$ iff $\psi$ is unsatisfiable.
\end{proof}

In this reduction, we do \emph{not} exclude $\phi^*$ from the constraint encoding to ensure $\Phi_C \neq \varnothing$. 
This differs from the algorithmic enumeration (Algorithm~\ref{alg:asp_verification}) where we \emph{do} exclude $\phi^*$ to enumerate alternative shortcut solutions.

\begin{theorem}[Shortcut Counting is \#P-complete]
\label{thm:complexity_sharpp}
    Given a constraint-based NSL problem $\mathcal{P} = (N, S, C, \phi^*, D)$ with $|N| = |S|$, computing $SM(C) = |\Phi_C| - 1$ is \emph{\#P-complete}.
\end{theorem}

\begin{proof} (Proof sketch)
\emph{Membership in \#P:} candidate mappings are verifiable in polynomial-time, so counting them is in \#P. 
\emph{\#P-Hardness:} using the same reduction as Theorem \ref{thm:complexity_conp}, each $\phi \neq \phi^*$ in $\Phi_C$ corresponds bijectively to a satisfying assignment of $\psi'$, so $SM(C)=\#\mathrm{SAT}(\psi')=\#\mathrm{SAT}(\psi)$.
Since \#SAT is \#P-complete \cite{valiant1979complexity}, counting $SM(C)$ is \#P-hard.
\end{proof}

\begin{theorem}[Minimal Repair over Candidate Pinning Constraints is NP-hard]
\label{thm:complexity_repair}
    Given a constraint-based NSL problem $\mathcal{P} = (N, S, C, \phi^*, D)$ and a library $\mathcal{L}$ of candidate pinning constraints, finding the smallest subset of $\mathcal{L}$ whose addition to $C$ achieves $SM^{\mathrm{all}}(C') = 0$ is \emph{NP-hard}.
\end{theorem}

\begin{proof} (Proof sketch)
We reduce Set Cover to the minimal repair problem.
Given universe $U = \{u_1, \ldots, u_n\}$ and collection $\mathcal{S}$, construct a constraint-based NSL problem with $n$ neural outputs, identity intended mapping, and no initial constraints (so $|\Phi_C| = n^n$, treating mappings as arbitrary functions rather than bijections).
Each set $S_j \in \mathcal{S}$ corresponds to a candidate constraint that pins $\phi(n_i) = \phi^*(n_i)$ for all $u_i \in S_j$.
Achieving $SM^{\mathrm{all}}(C') = 0$ with $\leq \ell$ constraints requires pinning every position, which corresponds exactly to covering $U$ with $\leq \ell$ sets.
\end{proof}

This hardness result explains why Algorithm~\ref{alg:repair} cannot guarantee optimality (Theorem~\ref{thm:repair_correctness}, Part 4), since finding the minimum number of repair constraints would itself require solving an NP-hard problem.

\section{Experiments}\label{sec:experiments}

We validate Algorithm~\ref{alg:asp_verification} (verification) and Algorithm~\ref{alg:repair} (greedy repair) on 8 domains from the \textit{RSBench} benchmark suite~\cite{bortolottiNeurosymbolicBenchmarkSuite2024}, covering visual arithmetic (MNIST-Half, MNIST-XOR, MNIST-EvenOdd, MNIST-Math), visual logic (CLE4EVR, Kandinsky), and autonomous driving (BDD-OIA, SDD-OIA). 
Each domain defines a constraint-based NSL problem $\mathcal{P} = (N, S, C, \phi^*, D)$ with varying numbers of neural outputs ($|N|$ from 2 to 21), concepts, and constraint structures. 
All experiments use \textit{clingo} with \texttt{--models=10,000} for bounded enumeration, averaged over 10 runs.
In deployment, exact enumeration is often unnecessary: shortcut-freeness can be checked by searching for \textit{any} alternative mapping, and shortcut analysis can be approximated via bounded enumeration (e.g., \texttt{--models=M}), randomized restarts, or incremental verification (Section~\ref{sec:practical_guide}).
We report results under two encodings.
The \emph{base} encoding enforces only that each neural output maps to exactly one concept (i.e., $\phi$ is a function), but allows multiple outputs to share the same concept label; this corresponds to $\Phi_C^{\mathrm{all}}$ in our notation.
The \emph{bijective} encoding additionally requires that each concept is used exactly once, corresponding to $\Phi_C^{\mathrm{bij}}$.
Our theoretical results assume bijectivity, but we include the base encoding to show how much ambiguity bijectivity alone resolves.
Table~\ref{tab:results} summarizes the results.

The experiments paint an empirical \emph{difficulty spectrum} that reflects our complexity classification.
Verification (Algorithm~\ref{alg:asp_verification}) is practical: even on the largest domains (21 neural outputs), bounded enumeration completes in under a second. 
Repair difficulty, in contrast, varies sharply by domain: MNIST tasks repair in 1--2 iterations, BDD-OIA requires 21 pinning constraints, and SDD-OIA resists all strategies tested, which is consistent with the NP-hardness of minimal repair (Theorem~\ref{thm:complexity_repair}). 
The complementary success of greedy and random strategies (greedy wins on BDD-OIA, random wins on CLE4EVR) suggests that hybrid strategies are a promising direction for future work. 

The approximate shortcut counting method employed by RSBench~\cite{bortolottiNeurosymbolicBenchmarkSuite2024} enumerates reasoning shortcuts by generating all possible ground-truth concept vectors exhaustively and encoding the problem as SAT. 
This is tractable for small domains (e.g., MNIST-Half: $5^5 = 3{,}125$ vectors) but infeasible for larger ones, e.g., Kandinsky's 18 ternary concepts yield $3^{18} \approx 387$ million vectors. 
We instead formulate shortcut enumeration as an ASP problem over minimal constraint sets.
For MNIST domains, arithmetic constraints directly encode the task rules without requiring training samples.
For non-MNIST domains (marked $^*$ in Tables~\ref{tab:results} and \ref{tab:repair}), we use 2--8 synthetic samples targeting distinct aspects of each classification rule; in Kandinsky we reduced the concept domain from 18 to 8 (details in supplementary material).
Since $C \subseteq C'$ implies $\Phi_{C'} \subseteq \Phi_C$ (monotonicity), our counts are upper bounds: $SM(C) = 0$ under minimal constraints guarantees shortcut-freeness under any superset, including the full training set.

\begin{table}[ht]
\centering
\caption{Reasoning shortcut detection: number of shortcuts and analysis time.
$^*$: Minimal synthetic samples,  $^\dagger$: Reduced concept domain (from 18 to 8), 
N: number of neural outputs, Bij.: bijective encoding, Time represents average over 10 runs, \textbf{0} indicates \texttt{UNSAT} (no shortcuts exist).}
\label{tab:results}
\begin{tabular}{@{}lrrccc@{}}
\toprule
\multirow{2}{*}{Task} & \multirow{2}{*}{N} & \multicolumn{2}{c}{Shortcuts} & \multicolumn{2}{c}{Time (ms)} \\
\cmidrule(lr){3-4} \cmidrule(lr){5-6}
 &  &  Base & Bij. & Base & Bij. \\
\midrule
MNIST-XOR & 2 & 1 & 1 & 353.3 & 301.7 \\
MNIST-Half & 5 & 2 & \textbf{{0}} & 291.2 & 290.7 \\
MNIST-EvenOdd & 10 & 35 & 1 & 299.9 & 299.4 \\
MNIST-Math & 4 & 4 & \textbf{{0}} & 302.5 & 301.4 \\
BDD-OIA$^*$ & 21 & $\ge10^4$ & $\ge10^4$ & 528.4 & 521.0 \\
SDD-OIA$^*$ & 21 & $\ge10^4$ & $\ge10^4$ & 527.5 & 523.2 \\
CLE4EVR$^*$ & 8 & $\ge10^4$ & 5,759 & 471.0 & 421.1 \\
Kandinsky$^*\dagger$ & 8 & $\ge10^4$ & 3,455 & 470.7 & 377.3 \\
\bottomrule
\end{tabular}
\end{table}

Algorithm~\ref{alg:asp_verification} successfully detects shortcuts across all 8 domains and fully enumerates valid mappings whenever the number of alternatives is below the 10,000-model cap; for the 4 domains where the count saturates at $\geq 10^4$, the reported value is a lower bound rather than an exhaustive count.
Enforcing bijectivity eliminates all shortcuts in 2 of 8 domains ($SM^{\mathrm{bij}} = 0$). 
This also shows that bijectivity combined with sufficiently strong constraints can achieve uniqueness, but is not universally sufficient.

\begin{table}[ht]
\centering
\caption{Repair comparison (base encoding): greedy vs.\ random repair.
I.: iterations, C.: constraints added, S.: success rate (\%), 
``--'' indicates failure to repair within $T = 100$ iterations.}
\label{tab:repair}
\begin{tabular}{@{}l ccc ccc@{}}
\toprule
\multirow{2}{*}{Task} & \multicolumn{3}{c}{Greedy Repair} & \multicolumn{3}{c}{Random Repair} \\
\cmidrule(lr){2-4} \cmidrule(lr){5-7}
 & I. & C. & S. & I. & C. & S. \\
\midrule
MNIST-XOR & 2.0 & 1.0 & 100\% & 2.0 & 1.0 & 100\% \\
MNIST-Half & 2.0 & 1.0 & 100\% & 2.0 & 1.0 & 100\% \\
MNIST-EvenOdd & 3.0 & 2.0 & 100\% & 3.0 & 2.0 & 100\% \\
MNIST-Math & 2.0 & 1.0 & 100\% & 2.0 & 1.0 & 100\% \\
BDD-OIA$^*$ & 22.0 & 21.0 & 100\% & -- & -- & 0\% \\
SDD-OIA$^*$ & -- & -- & 0\% & -- & -- & 0\% \\
CLE4EVR$^*$ & -- & -- & 0\% & 8.1 & 7.1 & 80\% \\
Kandinsky$^*\dagger$ & 9.0 & 8.0 & 100\% & 9.0 & 8.0 & 100\% \\
\bottomrule
\end{tabular}
\end{table}

Table~\ref{tab:repair} compares greedy and random repair on the base (non-bijective) encoding.
Greedy repair (Algorithm~\ref{alg:repair}) achieves full shortcut elimination in 6 of 8 domains, always terminating within the $k$-iteration bound of Theorem~\ref{thm:repair_correctness}.
The two strategies show no statistically significant difference overall, though they exhibit complementary strengths: greedy repair succeeds on BDD-OIA where random fails, while random repair succeeds on CLE4EVR where greedy does not.
SDD-OIA alone resists all strategies, retaining \(\geq 10^4\) shortcuts even with bijectivity enforced.
Its 21 neural outputs with weak binary constraints yield a nearly unconstrained solution space where single-position pinning eliminates too few mappings per iteration, which is consistent with the NP-hardness of minimal repair (Theorem~\ref{thm:complexity_repair}).

%%%%%%%%%%%%%%%%%%%%%%%%%%%%%%%%%%%%%%%%%%%%%%%%%%%%%%%%%%%%%%%%%%%%%%%%%%%%%%%%%%%%%%%%%%%%%%%%%%%

\section{Conclusion}\label{sec:conclusion}

We formalized reasoning shortcuts in neurosymbolic learning as a constraint satisfaction problem and established both theoretical foundations and practical tools.
On the theoretical side, we proved that discrimination is necessary for shortcut-freeness under bijective mappings (Theorem~\ref{thm:discrimination_necessary}), but demonstrated via explicit counterexample that it is insufficient even when the constraint graph is connected (Example~\ref{cex:discrimination_fails}).
We provided a complexity classification: deciding shortcut-freeness is coNP-complete, counting shortcuts is \#P-complete, and minimal repair is NP-hard (Theorems~\ref{thm:complexity_conp}--\ref{thm:complexity_repair}).
Sample complexity for disambiguation lies between $\lceil \log_r(k+1) \rceil$ and $|\Delta_C|$ label queries (Theorem~\ref{thm:sample_complexity}).
Our analysis applies to the constraint-based class of NeSy systems captured by Definition~\ref{def:nesy_problem}; characterizing analogous shortcut phenomena in systems without explicit logical constraints remains an interesting question for future work.

On the practical side, our ASP-based verification algorithm correctly detects all bijective shortcuts (Proposition~\ref{prop:asp_correctness}), and greedy repair converges in at most $k$ iterations (Theorem~\ref{thm:repair_correctness}).
Experiments across 8 RSBench domains validate both algorithms: bijectivity alone eliminates shortcuts in 2 of 8 domains, and greedy and random repair exhibit complementary strengths on the remainder, jointly resolving 7 of 8 domains.

Several directions remain open.
Most pressing is a complete characterization of sufficient conditions for uniqueness beyond trivial automorphism groups: formalizing when constraint connectivity and strength jointly guarantee shortcut-freeness.
Our bijective restriction ($|N| = |S|$, each concept used exactly once), while natural for balanced classification, excludes many practical scenarios; extending to non-bijective mappings and developing approximation algorithms for minimal repair are natural next steps.

\section*{Acknowledgements}
This work has been supported by JSPS KAKENHI Grant Number JP25K03190 and JST CREST Grant Number JPMJCR22D3.

\section*{AI Declaration}
The authors used a large language model (Claude) for improving the clarity of the writing. 
All technical content was developed and verified by the authors.

%% The file kr.bst is a bibliography style file for BibTeX 0.99c
\bibliographystyle{kr}
\bibliography{ref}

%%%%%%%%%%%%%%%%%%%%%%%%%%%%%%%%%%%%%%%%%%%%%%%%%%%%%%%%%%%%%%%%%%%%%%%%%%%%%%%%%%%%%%%%%%%%%%%%%%%
%% Supplementary Material (was appendix)
%%%%%%%%%%%%%%%%%%%%%%%%%%%%%%%%%%%%%%%%%%%%%%%%%%%%%%%%%%%%%%%%%%%%%%%%%%%%%%%%%%%%%%%%%%%%%%%%%%%

\clearpage
\appendix

\section{Worked Examples: Repair and Bijectivity}

\begin{example}[Repair Demonstration]
\label{ex:repair_demo}
We demonstrate greedy shortcut repair (Algorithm~\ref{alg:repair}) on the 4-node problem (Example~\ref{ex:simple_addition_intro}).

\emph{Initial State:}
\begin{itemize}
    \item Constraints: $C = \{C_1: \phi(n_0)+\phi(n_3)=3, C_2: \phi(n_1)+\phi(n_2)=3\}$
    \item Intended: $\phi^* = (0,1,2,3)$  
    \item Verification (Algorithm~\ref{alg:asp_verification}): $SM(C) = 7$ shortcuts
\end{itemize}

\emph{Detected shortcuts include:}
\begin{itemize}
    \item $\phi_2 = (0,2,1,3)$ $\quad$ swaps component 2
    \item $\phi_7 = (3,1,2,0)$ $\quad$ swaps component 1
    \item ... (5 more)
\end{itemize}

\emph{Running Algorithm~\ref{alg:repair} (greedy shortcut repair):}

\emph{Iteration 1:} Detect $\phi_2 = (0,2,1,3)$
\begin{itemize}
    \item Disagreement: $D_{sc} = \{n_1, n_2\}$ (since $\phi^*(n_1)=1 \neq 2=\phi_2(n_1)$)
    \item Add: $c_3: \phi(n_1) = 1$
    \item This forces $\phi(n_2) = 2$ (from $C_2: \phi(n_1)+\phi(n_2)=3$)
    \item Component 2 is now fully determined: $\{n_1 \mapsto 1, n_2 \mapsto 2\}$
    \item Component 1 remains free: $\{n_0, n_3\}$ can be $(0,3)$ or $(3,0)$
    \item New: $SM(C') = 1$ (2 valid bijections remain: $\phi^*$ and $\phi_7=(3,1,2,0)$)
\end{itemize}

\emph{Iteration 2:} Detect $\phi_7 = (3,1,2,0)$  
\begin{itemize}
    \item Disagreement: $D_{sc} = \{n_0, n_3\}$
    \item Add: $c_4: \phi(n_0) = 0$
    \item This forces $\phi(n_3) = 3$ (from $C_1$)
    \item Component 1 is now fully determined: $\{n_0 \mapsto 0, n_3 \mapsto 3\}$
    \item New: $SM(C'') = 0$ (shortcut-free)
\end{itemize}

Converged in 2 iterations ($\leq 7$ as in Theorem~\ref{thm:repair_correctness}).
\end{example}

\begin{example}[MNIST-Half: Bijectivity Ensures Uniqueness]
\label{ex:mnist_half_bijectivity_effect}
We demonstrate that the MNIST-Half constraints from Example~\ref{ex:mnist_half_nesy_problem}, when combined with bijectivity enforcement, actually achieve uniqueness.

\emph{Setup:}
\begin{itemize}
    \item Constraints: $C = \{C_1, C_2, C_3, C_4\}$ (as in Example~\ref{ex:mnist_half_nesy_problem})
    \item Intended: $\phi^* = (0, 1, 2, 3, 4)$
\end{itemize}

\emph{Analysis WITHOUT bijectivity:}

From $C_1, C_2$: $\phi(n_0) = 0, \phi(n_1) = 1$ (forced).

For $\{n_2, n_3, n_4\}$, allowing repetition:
\begin{itemize}
    \item $\phi^* = (0, 1, 2, 3, 4)$ --- bijective (valid)
    \item $\phi_1 = (0, 1, 3, 2, 3)$ --- overloads 3: $C_3: 3+2=5$, $C_4: 3+3=6$ (valid)
    \item $\phi_2 = (0, 1, 4, 1, 2)$ --- overloads 1: $C_3: 4+1=5$, $C_4: 4+2=6$ (valid)
\end{itemize}

Without bijectivity: $SM^{\mathrm{all}}(C) \geq 2$ (multiple non-bijective shortcuts).

\emph{Analysis WITH bijectivity (Algorithm~\ref{alg:asp_verification} enforcement):}

Must use $\{0,1,2,3,4\}$ exactly once. With $\phi(n_0)=0, \phi(n_1)=1$ forced, $\{n_2, n_3, n_4\}$ must partition $\{2,3,4\}$.

For $C_3$: $\phi(n_2) + \phi(n_3) = 5$ from $\{2,3,4\}$:
\begin{itemize}
    \item Only possibility: $\{\phi(n_2), \phi(n_3)\} = \{2, 3\}$ (since $4+1=5$ but $1 \notin \{2,3,4\}$)
\end{itemize}

For $C_4$: $\phi(n_2) + \phi(n_4) = 6$:
\begin{itemize}
    \item If $\phi(n_2) = 2$: then $\phi(n_4) = 4$ and $\phi(n_3) = 3$ (from $C_3$), which recovers $\phi^*$
    \item If $\phi(n_2) = 3$: then $\phi(n_4) = 3$, violating bijectivity
\end{itemize}

With bijectivity: $|\Phi_C^{\mathrm{bij}}| = 1$, $SM^{\mathrm{bij}}(C) = 0$, and it becomes shortcut-free.

In MNIST-Half, bijectivity + 4 arithmetic constraints suffice for uniqueness. 
However, bijectivity doesn't always eliminate shortcuts, see Example~\ref{ex:simple_addition_intro} where 8 bijective mappings exist, out of which 7 are shortcuts, even with bijectivity enforced. 
The difference is constraint strength: MNIST-Half has coupled arithmetic constraints linking all neural outputs in component 2, while Example~\ref{ex:simple_addition_intro} has weak sum constraints allowing independent permutations within disconnected components. 
This demonstrates that disconnection alone doesn't guarantee shortcuts when constraints are sufficiently restrictive.
\end{example}

%%%%%%%%%%%%%%%%%%%%%%%%%%%%%%%%%%%%%%%%%%%%%%%%%%%%%%%%%%%%%%%%%%%%%%%%

\section{Non-Bijective vs. Bijective Shortcuts}

\begin{example}[Bijectivity Eliminates Overloading, Not All Shortcuts]
\label{ex:bijectivity_comparison}
We compare two scenarios to show that bijectivity can eliminate overloading shortcuts, but is not sufficient for shortcut-freeness.

\emph{Scenario 1: Bijectivity Eliminates Shortcuts (MNIST-Half)}

From Example~\ref{ex:mnist_half_bijectivity_effect}: With 4 constraints and bijectivity, $SM^{\mathrm{bij}}(C) = 0$ (unique). 
Without bijectivity, $SM^{\mathrm{all}}(C) \geq 2$ (overloading allows shortcuts).

\emph{Scenario 2: Bijectivity is Insufficient (4-Node Example)}

From Example~\ref{ex:simple_addition_intro}: With 2 constraints and bijectivity, $SM^{\mathrm{bij}}(C) = 7$ (8 valid bijections). 
All shortcuts are bijective, and the disconnected constraint graph allows independent permutations.

Therefore, bijectivity can help eliminate certain non-bijective shortcuts, but does not by itself guarantee uniqueness.
Whether bijectivity suffices for a given constraint set depends on constraint strength: MNIST-Half achieves uniqueness despite disconnected components because its arithmetic constraints, together with bijectivity, fully determine each component; the 4-node example leaves independent transpositions unresolved.

For systems with disconnected constraint graphs, additional constraints or labeled examples are needed (Algorithm~\ref{alg:repair}, Section~\ref{sec:practical_guide}).
\end{example}

%%%%%%%%%%%%%%%%%%%%%%%%%%%%%%%%%%%%%%%%%%%%%%%%%%%%%%%%%%%%%%%%%%%%%%%%

\section{Worked Example: Active Learning}
\begin{example}[Uncertainty Sampling on 4-Node Problem]
\label{ex:uncertainty_sampling_demo}
Apply uncertainty sampling (Section~\ref{sec:practical_guide}) to the 4-node addition problem (Example~\ref{ex:simple_addition_intro}) with $SM(C) = 7$ shortcuts.

\emph{Initial State:} $|\Psi| = 8$ candidate bijections:
\begin{itemize}
    \item $\phi_1 = (0,1,2,3)$, $\phi_2 = (0,2,1,3)$ (component 1: $(0,3)$)
    \item $\phi_3 = (1,0,3,2)$, $\phi_4 = (1,3,0,2)$ (component 1: $(1,2)$)
    \item $\phi_5 = (2,0,3,1)$, $\phi_6 = (2,3,0,1)$ (component 1: $(2,1)$)
    \item $\phi_7 = (3,1,2,0)$, $\phi_8 = (3,2,1,0)$ (component 1: $(3,0)$)
\end{itemize}

\emph{Step 1: Compute disagreement}
\begin{itemize}
    \item $n_0$: Candidates use $\{0, 1, 2, 3\} \rightarrow d(n_0) = 4$
    \item $n_1$: Candidates use $\{0, 1, 2, 3\} \rightarrow d(n_1) = 4$
    \item $n_2$: Candidates use $\{0, 1, 2, 3\} \rightarrow d(n_2) = 4$
    \item $n_3$: Candidates use $\{0, 1, 2, 3\} \rightarrow d(n_3) = 4$
\end{itemize}

All positions have maximum disagreement. Select $n^* = n_0$ (arbitrary tie-breaking).

\emph{Step 2: Query and eliminate}

Query label: $\phi^*(n_0) = 0$

Eliminate all $\phi \in \Psi$ with $\phi(n_0) \neq 0$: Remove $\{\phi_3, \phi_4, \phi_5, \phi_6, \phi_7, \phi_8\}$

New $|\Psi| = 2$: $\{\phi_1, \phi_2\}$ (both have $\phi(n_0) = 0, \phi(n_3) = 3$)

\emph{Step 3: Second query}

Compute disagreement:
\begin{itemize}
    \item $n_0$: All agree ($\phi(n_0) = 0$) $ \rightarrow d(n_0) = 1$
    \item $n_1$: Candidates use $\{1, 2\}$   $ \rightarrow d(n_1) = 2$
    \item $n_2$: Candidates use $\{1, 2\}$   $ \rightarrow d(n_2) = 2$
    \item $n_3$: All agree ($\phi(n_3) = 3$) $ \rightarrow d(n_3) = 1$
\end{itemize}

Select $n^* = n_1$ (highest disagreement, or $n_2$).

Query label: $\phi^*(n_1) = 1$

New $|\Psi| = 1$: $\{\phi_1 = \phi^*\}$ (identified)

Uncertainty sampling identified $\phi^*$ in 2 queries, matching the lower bound $\lceil \log_4(8) \rceil = 2$ from Theorem~\ref{thm:sample_complexity}. 
The favorable structure (component 1 choices fully determine component 2) enables this optimal performance.
\end{example}

\section{Concrete ASP Encoding}

\begin{lstlisting}[caption=Example ASP encoding for Example~\ref{ex:simple_addition_intro}]
% Domain declaration for safety
val(0..3).

% Neural outputs
neural(n0). neural(n1). neural(n2). neural(n3).

% Concepts derived from value domain
concept(V) :- val(V).

% Bijectivity: exactly one neural output per concept
1 { maps_to(N,S) : concept(S) } 1 :- neural(N).
1 { maps_to(N,S) : neural(N) } 1 :- concept(S).

% Value extraction with domain guard
value(N, V) :- maps_to(N, V), val(V).

% Constraints: encode arithmetic with domain guards
:- value(n0, V0), value(n3, V3), val(V0), val(V3), V0 + V3 != 3.
:- value(n1, V1), value(n2, V2), val(V1), val(V2), V1 + V2 != 3.

% Exclude intended mapping phi* = (0,1,2,3)
% (used in Algorithm 1 for enumeration)
:- maps_to(n0,0), maps_to(n1,1), 
   maps_to(n2,2), maps_to(n3,3).

#show maps_to/2.
% When executed with enumeration mode, this will output 7 mappings.
\end{lstlisting}

\begin{lstlisting}[caption=ASP encoding for MNIST-Half (Example~\ref{ex:mnist_half_nesy_problem}) without bijectivity constraint]
% Domain declaration
val(0..4).

% Neural outputs
neural(n0). neural(n1). neural(n2). neural(n3). neural(n4).

% Concepts
concept(V) :- val(V).

% Mapping (allowing non-bijective)
1 { maps_to(N,S) : concept(S) } 1 :- neural(N).

% MNIST-Half constraints
:- maps_to(n0, V0), V0 + V0 != 0.
:- maps_to(n0, V0), maps_to(n1, V1), V0 + V1 != 1.
:- maps_to(n2, V2), maps_to(n3, V3), V2 + V3 != 5.
:- maps_to(n2, V2), maps_to(n4, V4), V2 + V4 != 6.

% Exclude intended mapping phi* = (0,1,2,3,4)
:- maps_to(n0,0), maps_to(n1,1), maps_to(n2,2), 
   maps_to(n3,3), maps_to(n4,4).

#show maps_to/2.
% When executed with enumeration mode, this will output 2 mappings.
\end{lstlisting}

\begin{lstlisting}[caption=ASP encoding for MNIST-Half (Example~\ref{ex:mnist_half_nesy_problem}) with bijectivity constraint]
% Domain declaration
val(0..4).

% Neural outputs
neural(n0). neural(n1). neural(n2). neural(n3). neural(n4).

% Concepts
concept(V) :- val(V).

% Bijectivity: exactly one neural output per concept
1 { maps_to(N,S) : concept(S) } 1 :- neural(N).
1 { maps_to(N,S) : neural(N) } 1 :- concept(S).

% MNIST-Half constraints
:- maps_to(n0, V0), V0 + V0 != 0.
:- maps_to(n0, V0), maps_to(n1, V1), V0 + V1 != 1.
:- maps_to(n2, V2), maps_to(n3, V3), V2 + V3 != 5.
:- maps_to(n2, V2), maps_to(n4, V4), V2 + V4 != 6.

% Exclude intended mapping phi* = (0,1,2,3,4)
:- maps_to(n0,0), maps_to(n1,1), maps_to(n2,2), 
   maps_to(n3,3), maps_to(n4,4).

#show maps_to/2.
% UNSATISFIABLE
\end{lstlisting}

\begin{lstlisting}[caption=ASP encoding for Example~\ref{cex:discrimination_fails} (modulo successor)]
% Domain declaration
val(0..2).

% Neural outputs
neural(n0). neural(n1). neural(n2).

% Concepts
concept(V) :- val(V).

% Bijectivity
1 { maps_to(N,S) : concept(S) } 1 :- neural(N).
1 { maps_to(N,S) : neural(N) } 1 :- concept(S).

% Value extraction
value(N, V) :- maps_to(N, V), val(V).

% Modulo successor constraints
:- value(n0, V0), value(n1, V1), val(V0), val(V1), 
   V1 != (V0 + 1) \ 3.
:- value(n1, V1), value(n2, V2), val(V1), val(V2), 
   V2 != (V1 + 1) \ 3.

% Exclude intended mapping phi* = (0,1,2)
:- maps_to(n0,0), maps_to(n1,1), maps_to(n2,2).

#show maps_to/2.
% When executed with enumeration mode, this will output 2 mappings.
\end{lstlisting}

\section{Full Proofs}

\subsection{Shortcut-Freeness is coNP-complete}

\begin{proof}
\emph{Membership in coNP:} A certificate that $SM(C) > 0$ is a valid mapping $\phi \neq \phi^*$ with $\phi \in \Phi_C$. 

Verification: 
(i) Checking $\phi \neq \phi^*$ takes $O(r)$ time. 
(ii) Checking $\phi \in \Phi_C$ requires evaluating constraints, which takes $O(|C| \cdot \text{poly}(r))$ time.
Both are polynomial, so the complement problem is in NP, thus $SM(C) = 0$ is in coNP.

\emph{Hardness (coNP-hard).} We reduce \textsc{UNSAT} to deciding $SM(C)=0$ given $\phi^*$.

\emph{Reduction construction:}
Let $\psi$ be a CNF formula over Boolean variables $\{x_1,\ldots,x_m\}$.
We reduce from $\psi$ via an auxiliary formula $\psi' = \psi \wedge \neg y$, where $y$ is a fresh variable not appearing in $\psi$. 
Note that $\#\mathrm{SAT}(\psi') = \#\mathrm{SAT}(\psi)$, and $\psi'$ is satisfiable iff $\psi$ is satisfiable. The fresh variable ensures that the all-true assignment never satisfies $\psi'$, which is what makes the bijection between alternative mappings and satisfying assignments work in the construction below.

Construct a constraint-based NSL problem with:
\begin{itemize}
    \item \emph{Neural outputs} represent literals over the variables of $\psi'$: $N=\{n_i,\bar{n}_i: i\in[m]\} \cup \{n_y, \bar{n}_y\}$.
    \begin{itemize}
        \item $n_i$ represents positive literal $x_i$
        \item $\bar{n}_i$ represents negative literal $\neg x_i$
    \end{itemize}
    Thus $|N| = 2(m+1)$.
    
    \item \emph{Concepts} represent truth values: $S=\{T_i,F_i: i\in[m]\} \cup \{T_y, F_y\}$. 
    \begin{itemize}
        \item $T_i$ represents ``variable $x_i$ is true''
        \item $F_i$ represents ``variable $x_i$ is false''
        \item $T_i$ and $F_i$ are mutually exclusive truth values for variable $x_i$
    \end{itemize}
    Thus $|S| = 2(m+1)$.
    
    \item \emph{Intended mapping} corresponds to the all-true assignment over $\{x_1,\ldots,x_m,y\}$:
    \begin{align*}
        \phi^*(n_i) &= T_i, \quad \phi^*(\bar{n}_i) = F_i 
        \quad \text{for } i \in [m] \\
        \phi^*(n_y) &= T_y, \quad \phi^*(\bar{n}_y) = F_y
    \end{align*}
\end{itemize}

\emph{Interpretation:} 
Any admissible bijection $\phi: N \to S$, i.e., any bijection satisfying the pair-consistency constraints below, encodes a truth assignment over $\{x_1,\ldots,x_m,y\}$, where variable $v$ is true iff $\phi(n_v) = T_v$.

\emph{Constraint encoding:}
Bijectivity is enforced by constraints (each neural output maps to exactly one concept, each concept used exactly once). 
We additionally enforce pair-consistency: for each variable $v\in\{x_1,\ldots,x_m,y\}$, the outputs $n_v$ and $\bar n_v$ may only map to the concepts $T_v$ and $F_v$. 
Together with bijectivity, this ensures that exactly one of the following holds:
$\phi(n_v)=T_v,\phi(\bar n_v)=F_v$ or $\phi(n_v)=F_v,\phi(\bar n_v)=T_v$.
Thus every admissible bijection corresponds uniquely to a Boolean assignment.
To detect deviations from $\phi^*$, introduce auxiliary atom $\mathit{alt}$:
\begin{align*}
    & \text{For each variable } v \in \{x_1,\ldots,x_m,y\}: \\
    & \quad \mathit{alt} \leftarrow \mathit{maps\_to}(n_v,F_v) \\
    & \quad \mathit{alt} \leftarrow \mathit{maps\_to}(\bar{n}_v,T_v)
\end{align*}
Note that $\mathit{alt}$ is false iff $\phi = \phi^*$, and true otherwise.

For each clause $D_j$ of $\psi'$ (including the unit clause $\neg y$), encode clause satisfaction so that $\mathit{satisfied}(D_j)$ holds exactly when the encoded assignment satisfies $D_j$. 
For each literal $\ell$ appearing in $D_j$, add the rule:
\begin{align*}
    & \mathit{satisfied}(D_j) \leftarrow \mathit{maps\_to}(n_v,T_v) 
        \quad \text{if } \ell = x_v \\
    & \mathit{satisfied}(D_j) \leftarrow \mathit{maps\_to}(\bar{n}_v,T_v) 
        \quad \text{if } \ell = \neg x_v
\end{align*}
Both rules fire exactly when their literal is satisfied: $x_v$ is true iff $\phi(n_v) = T_v$, and $\neg x_v$ is true iff $\phi(\bar{n}_v) = T_v$ (equivalently, $\phi(n_v) = F_v$).

Finally, enforce that whenever $\mathit{alt}$ is true, every clause must be satisfied:
\begin{align*}
    \bot \leftarrow \mathit{alt},\, \mathit{not}\ \mathit{satisfied}(D_j) 
    \quad \text{for each clause } D_j \text{ of } \psi'.
\end{align*}

\emph{Translation:} ``If the mapping differs from $\phi^*$ ($\mathit{alt}$ is true), then every clause of $\psi'$ must be satisfied by the encoded assignment.''

\emph{Observation:} 
By construction, $\phi^* \in \Phi_C$: 
when $\phi = \phi^*$, the atom $\mathit{alt}$ is false and the clause constraints fire vacuously. 
Note that $\phi^*$ encodes the all-true assignment, which fails the unit clause $\neg y$ of $\psi'$; this is 
fine because the clause constraints only fire when $\mathit{alt}$ is true.

Any admissible bijection $\phi \neq \phi^*$ makes $\mathit{alt}$ true, and is valid iff its encoded assignment satisfies every clause of $\psi'$.

\emph{Correctness:}
\begin{align*}
    & \exists \phi \neq \phi^* \text{ with } \phi \in \Phi_C \\
    & \Leftrightarrow \exists \text{ assignment over } \{x_1,\ldots,x_m,y\} 
      \text{ satisfying } \psi' \\
    & \Leftrightarrow \psi' \text{ is satisfiable} \\
    & \Leftrightarrow \psi \text{ is satisfiable.}
\end{align*}

Therefore, $\psi$ is unsatisfiable $\Leftrightarrow$ $SM(C) = 0$.
The reduction produces a constraint set of size $O(|\psi|)$ (linear in the number of literal occurrences in $\psi$), and is computable in polynomial time.
Thus the problem is coNP-hard.
Combined with membership in coNP, deciding $SM(C) = 0$ is coNP-complete.
\end{proof}

\subsection{Shortcut Counting is \#P-complete}

\begin{proof}

\emph{Membership in \#P:} 
The function $SM(C)$ counts bijective mappings satisfying constraints in $C$ that are distinct from $\phi^*$.

For any candidate mapping $\phi: N \to S$, we can verify in polynomial time:
\begin{itemize}
    \item Bijectivity: Check each concept used exactly once ($O(r)$ time)
    \item Constraint satisfaction: Evaluate all constraints under $\phi$ 
    ($O(|C| \cdot \mathrm{poly}(r))$ time)
    \item Non-identity: Check $\phi \neq \phi^*$
\end{itemize}

Since these witnesses are polynomial-time verifiable and we are counting them, $SM(C)$ is in \#P.

\emph{Hardness (\#P-hard):} We reduce \#SAT (counting satisfying assignments of Boolean formulas) to computing $SM(C)$.

\emph{Reduction construction:} Given CNF formula $\psi$ over variables $\{x_1,\ldots,x_m\}$, we apply the same reduction as in Theorem~\ref{thm:complexity_conp}, introduce a fresh variable $y$ and set $\psi' = \psi \wedge \neg y$.
Construct a constraint-based NSL problem with:
\begin{itemize}
    \item Neural outputs: $N=\{n_i,\bar{n}_i: i\in[m]\} \cup \{n_y, \bar{n}_y\}$.
    \item Concepts: $S=\{T_i,F_i: i\in[m]\} \cup \{T_y, F_y\}$. 
    \item Intended mapping: $\phi^*(n_v) = T_v$, $\phi^*(\bar{n}_v) = F_v$ (the all-true assignment over the extended variable set)
    \item Constraints: enforcing bijectivity, pair-consistency, and satisfaction of every clause of $\psi'$ when $\mathit{alt}$ is true.
\end{itemize}

\emph{Counting correspondence:} 

Each bijection $\phi \neq \phi^*$ in $\Phi_C$ corresponds to a satisfying truth assignment of $\psi'$:
\begin{itemize}
    \item $\phi^*$ itself always satisfies $C$ (as shown in Theorem~\ref{thm:complexity_conp}; clause constraints fire vacuously when $\mathit{alt}$ is false)
    \item Each $\phi \neq \phi^*$ encodes an assignment over $\{x_1,\ldots,x_m,y\}$ where variable $v$ is true iff $\phi(n_v) = T_v$
    \item $\phi \in \Phi_C$ iff the encoded assignment satisfies all clauses of $\psi'$
\end{itemize}

Therefore:
\begin{align*}
    SM(C) &= |\Phi_C| - 1 \\
    &= |\{\phi \in \Phi_C : \phi \neq \phi^*\}| \\
    &= |\{\text{satisfying assignments of } \psi'\}| \\
    &= \#\mathrm{SAT}(\psi') = \#\mathrm{SAT}(\psi).
\end{align*}

Since \#SAT is \#P-complete \cite{valiant1979complexity} and this reduction is polynomial in the size of the input formula $\psi$, computing $SM(C)$ is \#P-hard. Combined with \#P membership, computing $SM(C)$ is \#P-complete.

\end{proof}

\subsection{Minimal Repair over Candidate Pinning Constraints is NP-hard}

\begin{proof}
\emph{Problem:} Given $C$, $\phi^*$, budget $\ell$, and a fixed library of candidate constraints, determine whether we can add $\leq \ell$ candidate constraints to achieve $SM^{\mathrm{all}}(C') = 0$. 
Here $SM^{\mathrm{all}}$ denotes shortcut multiplicity over all mappings, without enforcing bijectivity.

\emph{Background:} The classical \textsc{Set Cover} problem asks: given universe $U$ and collection of subsets $\mathcal{S} = \{S_1, \ldots, S_p\}$ with $S_i \subseteq U$, can we select $\leq \ell$ sets whose union equals $U$? This problem is NP-complete.

\emph{Reduction from Set Cover:} 

\emph{Intuition:} Each element of the universe $U$ corresponds to a ``degree of freedom'' in our constraint-based NSL problem. 
Each set $S_j \in \mathcal{S}$ corresponds to a potential constraint that eliminates certain mappings. 
Covering all elements corresponds to eliminating all shortcuts.

\emph{Construction:} Given universe $U = \{u_1, \ldots, u_n\}$ with $n \geq 2$, and collection $\mathcal{S} = \{S_1, \ldots, S_p\}$:

\begin{itemize}
    \item \emph{Neural outputs:} $N = \{n_1, \ldots, n_n\}$ (one per universe element)
    \item \emph{Concepts:} $S = \{0, 1, \ldots, n-1\}$
    \item \emph{Initial constraints:} $C = \{\}$ (empty); bijectivity is not enforced in this reduction.
    \begin{itemize}
        \item Each $\phi: N \to S$ is unconstrained, so $|\Phi_C^{\mathrm{all}}| = |S|^{|N|} = n^n$
        \item Thus $SM^{\mathrm{all}}(C) = n^n - 1$ (exponentially many shortcuts)
    \end{itemize}
    \item \emph{Intended mapping:} $\phi^* = \{n_i \mapsto i-1\}$ (identity: $n_1 \mapsto 0, n_2 \mapsto 1, \ldots$)
    \item \emph{Candidate constraints:} For each set $S_j \in \mathcal{S}$, define:
    \begin{align*}
        c_j : \bigwedge_{u_i \in S_j} \phi(n_i) = \phi^*(n_i)
    \end{align*}
    This constraint pins each neural output $n_i$ with $u_i \in S_j$ to its intended value $\phi^*(n_i)$.
\end{itemize}

\emph{Key correspondence:}

Without bijectivity, an unpinned position $n_i$ can independently take any value in $S$, so $\phi(n_i) \neq \phi^*(n_i)$ is available as a deviation. 
Hence $SM^{\mathrm{all}}(C') = 0$ iff every position is pinned by some selected $c_j$:
\begin{itemize}
    \item Adding constraint $c_j$ eliminates all mappings that differ from $\phi^*$ on any neural output $n_i$ where $u_i \in S_j$.
    \item For uncovered $u_i$, the mapping $\phi(n_i) = (\phi^*(n_i) + 1) \bmod n$ and $\phi(n_k)=\phi^*(n_k)$ for all $k\neq i$ is a valid shortcut, so $SM^{\mathrm{all}}(C') > 0$.
    \item Therefore $SM^{\mathrm{all}}(C') = 0$ iff selected sets cover $U$.
\end{itemize}

\emph{Correctness:}
\begin{align*}
    & \exists R \subseteq \{c_1,\ldots,c_p\} \text{ with } |R| \leq \ell \text{ and } SM^{\mathrm{all}}(C \cup R) = 0 \\
    & \Leftrightarrow \exists \text{ selection of } \leq \ell \text{ constraints from } \{c_1, \ldots, c_p\} \\
    & \qquad \text{forcing } \phi(n_i) = \phi^*(n_i) \text{ for all } i \\
    & \Leftrightarrow \exists \text{ selection of } \leq \ell \text{ sets from } \mathcal{S} \text{ covering } U
\end{align*}

The minimum number of candidate constraints needed to achieve $SM^{\mathrm{all}}(C') = 0$ exactly equals the minimum set cover size.
Since Set Cover is NP-complete, the decision version of minimal repair over candidate pinning constraints is NP-hard.
\end{proof}

\section{Additional Label Selection Strategies and Practitioner's Guide}

\subsection{Random Sampling}

\begin{observation}[Random Sampling -- Empirical]
\label{obs:random_sampling}
If examples are selected uniformly at random from $N$, the expected number of labeled examples appears empirically to be approximately:
$$\mathbb{E}[m] \approx \frac{|N|}{|\Delta_C|} \cdot k \log_2 k$$
This assumes: 
(1) disagreement positions are hit with probability $|\Delta_C|/|N|$, and 
(2) once informative examples are found, coupon-collector analysis suggests $O(k \log_2 k)$ examples suffice.
\end{observation}

\subsection{Greedy Disambiguation}

\emph{Strategy:} At each step, select the example that eliminates the most shortcuts:
\begin{enumerate}
    \item For each $n \in N$ and $s \in S$, compute: $e(n,s) = |\{\phi \in \Psi : \phi(n) \neq s\}|$
    \item Select $(n^*, s^*) = \arg\max_{n,s} e(n,s)$
    \item Query label $(n^*, s^*)$ and verify $s^* = \phi^*(n^*)$
    \item Remove all $\phi \in \Psi$ with $\phi(n^*) \neq s^*$
\end{enumerate}

This strategy explicitly maximizes shortcuts eliminated per query, closely approximating the oracle.

\begin{observation}[Greedy Performance -- Information-Theoretic]
\label{obs:greedy_performance}
Greedy disambiguation typically achieves $O(k \log_2 k)$ or better queries when shortcuts have diverse disagreement patterns (information-theoretic bound). 
However, we do not provide a formal worst-case guarantee, as performance depends on the structure of $\Phi_C$.
\end{observation}

\subsection{Strategy Comparison}

The oracle lower bound is $\lceil \log_r(k+1) \rceil$ from 
Theorem~\ref{thm:sample_complexity}, with upper bound $|\Delta_C|$.
Uncertainty sampling (Proposition~\ref{prop:uncertainty_bound}) achieves 
$O(\min(k, |\Delta_C|))$ queries; greedy disambiguation typically achieves 
$O(k \log_2 k)$; and random sampling requires 
$O(|N|/|\Delta_C| \cdot k \log_2 k)$ expected queries.
All bounds except uncertainty sampling are information-theoretic 
without formal worst-case guarantees.

\subsection{Practitioner's Guide}
\label{sec:practitioners_guide_appendix}

For practitioners using existing neurosymbolic frameworks (DeepProbLog, 
Logic Tensor Networks, NeurASP, etc.), we recommend the following workflow.

\emph{Step 1: Export Constraints.}
Convert logical rules into ASP format, define the intended mapping $\phi^*$, and identify neural outputs $N$ and concepts $S$ with $|N| = |S|$.

\emph{Step 2: Verify Shortcut-Freeness.}
Run Algorithm~\ref{alg:asp_verification}.
If the result is \texttt{SHORTCUT-FREE}, the unique valid mapping is confirmed. 
Otherwise, proceed to Step~3.

\emph{Step 3: Eliminate Shortcuts.}
Choose one of three approaches:
(A)~\emph{Automated repair}: use Algorithm~\ref{alg:repair}, which converges 
in $\leq k$ iterations;
(B)~\emph{Active learning}: use uncertainty sampling 
(Proposition~\ref{prop:uncertainty_bound}), requiring $\leq \min(k, |\Delta_C|)$ labels;
(C)~\emph{Manual}: analyze detected shortcuts and add domain-specific constraints.

\emph{Step 4: Verify Improvement.}
Re-run Algorithm~\ref{alg:asp_verification} to confirm shortcut elimination.

\section{Experimental Domain Details}

Table~\ref{tab:domain_details} summarizes the encoding used for each experimental domain, including any simplifications relative to the original RSBench specification~\cite{bortolottiNeurosymbolicBenchmarkSuite2024}.

\begin{table}[ht]
\centering
\caption{Domain encoding details.
S.: Number of sample data, 
Org. $|N|$: original number of neural outputs in RSBench before reduction, 
``Constraint-only'' means arithmetic rules are encoded directly without training samples.
}
\label{tab:domain_details}
\begin{tabular}{@{}lcccl@{}}
\toprule
Domain & $|N|$ & S. & Org. $|N|$ & Notes \\
\midrule
MNIST-XOR & 2 & 0 & 2 & Constraint-only \\
MNIST-Half & 5 & 0 & 5 & Constraint-only \\
MNIST-EvenOdd & 10 & 0 & 10 & Constraint-only \\
MNIST-Math & 4 & 0 & 4 & Constraint-only \\
BDD-OIA$^*$ & 21 & 8 & 21 & Synthetic samples \\
SDD-OIA$^*$ & 21 & 8 & 21 & Synthetic samples \\
CLE4EVR$^*$ & 8 & 4 & 8 & Synthetic samples \\
Kandinsky$^*\dagger$ & 8 & 2 & 18 & Reduced domain \\
\bottomrule
\end{tabular}
\end{table}

\emph{MNIST domains.}
Arithmetic constraints (e.g., $\phi(n_0) + \phi(n_1) = 1$) capture all information relevant to shortcut detection. 
No training samples are needed and no domain reduction is applied.

\emph{BDD-OIA / SDD-OIA.}
The full 21 concepts are preserved. 
8 hand-crafted symbolic samples cover key logical rules (traffic light states, obstacle presence, lane availability). 
This may miss constraints arising only from rare co-occurrences in the full dataset, potentially overestimating shortcut counts.

\emph{Kandinsky.}
The concept domain is reduced from 18 (3 figures $\times$ 3 objects $\times$ 2 attributes) to 8 (2 figures $\times$ 2 objects $\times$ 2 attributes) for ASP tractability. 
2 synthetic samples encode the classification rule. 
Shortcut counts reflect the reduced domain and represent a lower bound on the full 18-concept problem.

\emph{CLE4EVR.}
All 8 concepts (2 objects $\times$ 4 attributes) are preserved. 
4 synthetic samples encode the classification rule: 
\begin{align*}
    \mathrm{And}(
    & \mathrm{Eq}(\mathit{color}\_1, \mathit{color}\_2), \\
    & \mathrm{Eq}(\mathit{shape}\_1, \mathit{shape}\_2), \\
    & \mathrm{Eq}(\mathit{material}\_1, \mathit{material}\_2))
\end{align*}

\end{document}